\documentclass[11pt,letterpaper]{article}

\usepackage[top=1in,bottom=1in,left=1in,right=1in]{geometry}

\usepackage{hyperref}       
\usepackage{url}            
\usepackage{nicefrac}       
\usepackage{microtype}      
\usepackage{xcolor}         
\usepackage{pdfpages}
\usepackage[nottoc]{tocbibind}
\usepackage{comment}
\usepackage{url}
\usepackage{amsthm}
\usepackage{float}
\usepackage{booktabs}
\usepackage{bm}
\usepackage{float}
\usepackage{algorithm2e} 
\usepackage{enumitem}

\usepackage{microtype}
\usepackage{multirow}
\usepackage{rotating}
\usepackage{booktabs} 
\usepackage{natbib}
\usepackage{amsmath,amsfonts,amsthm,bm,amssymb,mathtools}
\usepackage{enumitem}
\usepackage{nicefrac}   
\usepackage{tikz}
\usepackage{subfigure}
\usepackage{graphicx}
\usepackage{caption}
\usepackage{url}                  
\usepackage{nicefrac}     
\usepackage{microtype}          
\usepackage{listings}
\usepackage{algorithm2e} 
\usepackage{comment}
\usepackage{arydshln}
\usepackage{comment}
\usepackage{hyperref}
\usepackage{dsfont}
\usepackage{float}

\newtheorem*{rep@theorem}{\rep@title}
\newcommand{\newreptheorem}[2]{%
	\newenvironment{rep#1}[1]{%
		\def\rep@title{#2 \ref{##1}}%
		\begin{rep@theorem}}%
		{\end{rep@theorem}}}
\makeatother
\newtheorem{theorem}{Theorem}
\newreptheorem{theorem}{Theorem}
\newtheorem{corollary}[theorem]{Corollary}
\newtheorem{lemma}[theorem]{Lemma}
\newreptheorem{lemma}{Lemma}
\newtheorem{definition}{Definition}

\def\E{{\mathbb{E}}} 

\usepackage[textsize=tiny]{todonotes}

\DeclareMathOperator*{\argsup}{arg\,sup}
\DeclareMathOperator*{\argmin}{arg\,min}
\DeclareMathOperator*{\sign}{sign}

\newcommand{\lmod}{\left\lvert \left\lvert}
\newcommand{\rmod}{\right\rvert \right\rvert}

\newcommand{\sqbr}[1]{\left[ #1 \right]} 
\newcommand{\crbr}[1]{\left\{#1\right\}} 
\newcommand{\nrbr}[1]{\left( #1 \right)} 
\newcommand{\norm}[1]{\left\| #1 \right\|}

\newcommand{\ip}[2]{\langle #1, #2 \rangle}

\title{On Exact Solutions of the Inner Optimization Problem of Adversarial Robustness}

\author{
	Deepak Maurya \\ 
	Department of Computer Science\\
	Purdue University \\ 
	West Lafayette, Indiana, USA \\ 
	\texttt{dmaurya@purdue.edu} \\
        \and 
        Adarsh Barik \\ 
	Institute of Data Science\\
	National University of Singapore \\ 
	Singapore \\ 
	\texttt{abarik@nus.edu.sg} \\
	\and
	Jean Honorio \\ 
	Department of Computer Science\\
	Purdue University \\ 
	West Lafayette, Indiana, USA \\ 
	\texttt{jhonorio@purdue.edu} \\
}
\date{}
\begin{document}
\maketitle

\begin{abstract}
In this work, we propose a robust framework that employs adversarially robust training to safeguard the ML models against perturbed testing data. Our contributions can be seen from both computational and statistical perspectives. 
  Firstly, from a \textit{computational/optimization} point of view, we derive the ready-to-use exact solution for several widely used loss functions with a variety of norm constraints on adversarial perturbation for various supervised and unsupervised ML problems, including regression, classification, two-layer neural networks, graphical models, and matrix completion. The solutions are either in closed-form, or an easily tractable optimization problem such as 1-D convex optimization, semidefinite programming, difference of convex programming or a sorting-based algorithm. Secondly, from \textit{statistical/generalization} viewpoint, using some of these results, we derive novel bounds of the adversarial Rademacher complexity for various problems, which entails new generalization bounds. Thirdly, we perform some sanity-check experiments on real-world datasets for supervised problems such as regression and classification, as well as for unsupervised problems such as matrix completion and learning graphical models, with very little computational overhead.
\end{abstract}

\section{Introduction}
\label{sec:intro}
Machine learning models are used in a wide variety of applications, such as image classification, speech recognition, and self-driving vehicles. The models employed in these applications can achieve a very high training time accuracy but can fail spectacularly in making trustworthy predictions on test data. 
Thus, it becomes important for existing machine learning models to be adversarially robust to avoid poor performance and have better generalization on test data. 
Our contribution in this work encompasses both the \emph{computational/optimization} and \emph{statistical/generalization} aspects of adversarial training for various supervised and unsupervised learning problems.

Firstly, from a \emph{computational/optimization} perspective, we focus on deriving \emph{the exact optimal solution} of the inner maximization optimization arising in adversarial training.  In particular, we provide ready-to-use results for a wide variety of loss functions and various norm constraints (See Table~\ref{tab:all_results_intro}). Moreover, unlike other domain specific works such as  \citep{jia2017adversarial,li2016understanding,belinkov2017synthetic,ribeiro2018anchors} in natural language processing and \citep{hendrycks2021natural,alzantot2018generating} in computer vision, we aim to provide an adversarially robust training model which covers a large class of machine learning problems.

\begin{table*}
	\caption{A summary of our results for various loss functions and norm constraints which are used in a wide variety of applications.}
	\label{tab:all_results_intro}
	{\footnotesize
	\begin{tabular}{@{\hspace{0.02in}}l@{\hspace{0.05in}}l@{\hspace{0.05in}}p{0.9in}@{\hspace{0.05in}}l@{\hspace{0.05in}}p{1.1in}@{\hspace{0.05in}}l@{\hspace{0.05in}}} 
		\toprule 
		&Problem &  Loss function & Norm constraint & Prior results &Our solution \\
		\midrule
		\multirow{3}{*}{\begin{turn}{90} {\footnotesize Warm up}\end{turn}}&Regression & Squared loss & Any norm & Euclidean norm \citep{xu2008robust}  & Closed form, Theorem \ref{thm:reg_sqerr}\\
		&Classification &  Logistic loss & Any norm & Euclidean norm \citep{liu2020loss}& Closed form, Theorem \ref{thm:logit} \\
		&Classification & Hinge loss & Any norm & None & Closed form, Theorem \ref{thm:hinge} \\ [0.1cm] \hline 
		\multirow{7}{*}{\begin{turn}{90} {\footnotesize Main results}\end{turn}} &Classification & Two-Layer NN with convex and nonconvex activations & Any norm & ReLU activation \citep{Awasthi_Mao_Mohri_Zhong_2024} & Difference of convex functions, Theorem \ref{thm:2nns} \\
		&Graphical Models & Log-likelihood & Euclidean & None & 1-D optimization, Theorem \ref{thm:Gaussmdl_l2} \\
		&Graphical Models & Log-likelihood & Entry-wise $\ell_{\infty}$ & None  &  Semidefinite programming, Theorem \ref{thm:Gaussmdl_linf} \\
		&Matrix Completion & Squared loss  & Frobenius & None & Closed form, Theorem \ref{thm:mat_compl_fro} \\
		&Matrix Completion & Squared loss  & Entry-wise $\ell_{\infty}$& None  & Closed form, Corollary \ref{cor:matrixcompl_linf} \\
		&Max-Margin MC & Hinge loss  & Frobenius& None &  Sorting based algorithm, Theorem \ref{thm:maxmarginmat_fro}\\
		&Max-Margin MC & Hinge loss  & Entry-wise $\ell_{\infty}$& None & Closed form, Corollary \ref{cor:maxmarginmat_linf} \\
		\bottomrule
	\end{tabular}}
\end{table*}   

From a \textit{statistical/generalization} perspective, our contributions include providing upper and lower bounds for adversarial Rademacher complexity ($\mathcal{O}\nrbr{\nicefrac{1}{\sqrt{n}}}$). These bounds are based on the results summarized in Table \ref{tab:all_results_intro}. By analyzing the adversarial Rademacher complexity, we can infer the generalization aspect discussed in the introduction of Section \ref{sec:rademacher}. In this regard, we propose novel adversarial Rademacher complexity bounds for various ML problems like linear regression (applicable for any general norm), matrix completion, and max-margin matrix completion. Additionally, we propose bounds in Theorem \ref{thm:RC_linclass} for linear classifiers, which can be seen as a generalization of the $\ell_{\infty}$ norm \citep{yin2019rademacher} or any specific $p$-norm \citep{awasthi2020adversarial}.

To make the above-discussed computational and statistical contributions, we propose closed-form solutions or an easily tractable optimization problem for computing the adversarial perturbation in these multiple ML problems. These solutions summarized in Table \ref{tab:all_results_intro} are in closed form or can be obtained from a one-dimensional dual problem, semidefinite programming (SDP), or a sorting-based algorithm which are quite unexpected and novel in the context of adversarial perturbation. The solution to the inner maximization for various problems summarized in Table \ref{tab:all_results_intro} allows to get better robustness in a computationally cheap manner. 

\paragraph{Our Contributions.} Broadly, we make the following contributions through this work:
\begin{itemize} 
	\item \textbf{Adversarially robust formulation:} We use the adversarially robust training framework of \citet{yin2019rademacher}  using worst case adversarial attacks. Under this framework, we analyze several supervised and unsupervised ML problems, including regression, classification, two-layer neural networks, graphical models, and matrix completion. The solutions are either in closed-form, 1-D optimization, semidefinite programming, difference of convex programming, or a sorting-based algorithm. 
	
	\item \textbf{Computational/Optimization front:} We provide a plug-and-play solution which can be easily integrated with existing training algorithms. This is a boon for practitioners who can incorporate our method in their existing models with minimal changes. As a conscious design choice, we provide computationally cheap solutions for our optimization problems. 
	\item \textbf{Statistical/Generalization front:} On the theoretical front, we provide a systematic analysis for several loss functions and norm constraints, which are commonly used in applications across various domains. Table~\ref{tab:all_results_intro} provides a summary of our findings in a concise manner. Using some of these results, we further provide novel lower and upper bounds of the adversarial Rademacher complexity ($\mathcal{O}\nrbr{\nicefrac{1}{\sqrt{n}}}$) for various problems, which entails novel generalization bounds.
	\item \textbf{Real world experiments:} We further perform some sanity-check experiments on several real-world datasets. We show that our plug-and-play solution performs better (most of the time, in terms of test metrics and/or runtime) as compared to the fast gradient sign method (FGSM) \citep{goodfellow2014explaining}, projected gradient descent (PGD), and TRADES \citep{zhang2019theoretically}.
 
\end{itemize}
\section{Preliminaries}
\label{sec:prelim}
For any general prediction problem in machine learning (ML), consider we have $n$ samples of $\left(\mathbf{x}, y \right)$, where we try to predict $\mathbf{y} \in \mathcal{Y}$ from $\mathbf{x} \in \mathcal{X}$ using the function $f: \mathcal{X} \rightarrow \mathcal{Y}$. Assuming that the function $f$ can be parameterized by some parameter $\mathbf{w}$, we minimize a loss function $l(\mathbf{x}, y, \mathbf{w})$ to obtain an estimate of $\mathbf{w}$ from $n$ samples: 
\begin{align}
    \label{eq:clean opt prob}
    \hat{\mathbf{w}} = \argmin_{w} \frac{1}{n}\sum_{i = 1}^n l(\mathbf{x}^{(i)}, y^{(i)}, \mathbf{w})
\end{align}
where $(\mathbf{x}^{(i)}, y^{(i)})$ represents the $i^{\text{th}}$ sample.
Intuitively, with no prior information on the shift of the testing distribution, it makes sense to be prepared for absolutely worst-case scenarios. We incorporate this insight formally in our proposed \emph{adversarially robust training model}. At each iteration of the training algorithm, we generate worst-case adversarial samples using the current model parameters and ``clean'' training data within the bounds of a maximum norm. The model parameters are updated using these worst-case adversarial samples, and the next iteration is performed. 
Figure~\ref{fig:train domain} and Figure ~\ref{fig:adversarial train domain} provide a geometric interpretation of our training process.

\begin{figure*}
	\centering
 \begin{minipage}{0.45\textwidth}
		\centering
\begin{tikzpicture}
	\draw[ -latex] (0, 0) -- (5, 0);
	\draw[ -latex] (0, 0) -- (0, 5);
	\draw[ -latex] (0, 0) -- (4, 4);
	\draw[thick] (0.7,1.2) rectangle (3.3, 3.2);
	\draw[thick] (1.2,1.7) rectangle (3.8, 3.7);
	\draw[thick] (0.7, 1.2) -- (1.2, 1.7); \draw[thick] (3.3, 3.2) -- (3.8, 3.7); \draw[thick] (0.7, 3.2) -- (1.2, 3.7); \draw[thick] (3.3, 1.2) -- (3.8, 1.7);	
	\filldraw (0.7, 1.2) circle (3pt); 	\filldraw (1.2, 1.7) circle (3pt); 	\filldraw (3.3, 3.2) circle (3pt); \filldraw (3.8, 3.7) circle (3pt); \filldraw (0.7, 3.2) circle (3pt); \filldraw (1.2, 3.7) circle (3pt); \filldraw (3.3, 1.2) circle (3pt); \filldraw (3.8, 1.7) circle (3pt);
	\filldraw (1.7, 1.3) circle (3pt); \filldraw (2.5, 2.0) circle (3pt); \filldraw (1.9, 2.3) circle (3pt); \filldraw (2.6, 2.7) circle (3pt); \filldraw (1.4, 3) circle (3pt); \filldraw (2, 2.8) circle (3pt); \filldraw (3.4, 2.3) circle (3pt); \filldraw (2.6, 1.4) circle (3pt);
\end{tikzpicture}   
		\caption{Training domain}
		\label{fig:train domain}
 \end{minipage}
  \begin{minipage}{0.45\textwidth}
		\centering
\begin{tikzpicture}
	\draw[ -latex] (0, 0) -- (5, 0);
	\draw[ -latex] (0, 0) -- (0, 5);
	\draw[ -latex] (0, 0) -- (4, 4);
	\draw[thick] (0.7,1.2) rectangle (3.3, 3.2);
	\draw[thick] (1.2,1.7) rectangle (3.8, 3.7);
	\draw[thick] (0.7, 1.2) -- (1.2, 1.7); \draw[thick] (3.3, 3.2) -- (3.8, 3.7); \draw[thick] (0.7, 3.2) -- (1.2, 3.7); \draw[thick] (3.3, 1.2) -- (3.8, 1.7);
	\filldraw[green!20] (0.7, 1.2) circle (6pt); 	\filldraw[green!20] (1.2, 1.7) circle (6pt); 	\filldraw[green!20] (3.3, 3.2) circle (6pt); \filldraw[green!20] (3.8, 3.7) circle (6pt); \filldraw[green!20] (0.7, 3.2) circle (6pt); \filldraw[green!20] (1.2, 3.7) circle (6pt); \filldraw[green!20] (3.3, 1.2) circle (6pt); \filldraw[green!20] (3.8, 1.7) circle (6pt);
	\filldraw[green!20] (1.7, 1.3) circle (6pt); \filldraw[green!20] (2.5, 2.0) circle (6pt); \filldraw[green!20] (1.9, 2.3) circle (6pt); \filldraw[green!20] (2.6, 2.7) circle (6pt); \filldraw[green!20] (1.4, 3) circle (6pt); \filldraw[green!20] (2, 2.8) circle (6pt); \filldraw[green!20] (3.4, 2.3) circle (6pt); \filldraw[green!20] (2.6, 1.4) circle (6pt);	
	\filldraw (0.7, 1.2) circle (3pt); 	\filldraw (1.2, 1.7) circle (3pt); 	\filldraw (3.3, 3.2) circle (3pt); \filldraw (3.8, 3.7) circle (3pt); \filldraw (0.7, 3.2) circle (3pt); \filldraw (1.2, 3.7) circle (3pt); \filldraw (3.3, 1.2) circle (3pt); \filldraw (3.8, 1.7) circle (3pt);
	\filldraw (1.7, 1.3) circle (3pt); \filldraw (2.5, 2.0) circle (3pt); \filldraw (1.9, 2.3) circle (3pt); \filldraw (2.6, 2.7) circle (3pt); \filldraw (1.4, 3) circle (3pt); \filldraw (2, 2.8) circle (3pt); \filldraw (3.4, 2.3) circle (3pt); \filldraw (2.6, 1.4) circle (3pt);
	\draw[dashed, thick] (0.5,1) rectangle (3.5, 3.4);
	\draw[dashed, thick] (1,1.5) rectangle (4, 3.9);
	\draw[dashed, thick] (0.5, 1) -- (1, 1.5); \draw[dashed, thick] (3.5, 3.4) -- (4, 3.9); \draw[dashed, thick] (0.5, 3.4) -- (1, 3.9); \draw[dashed, thick] (3.5, 1) -- (4, 1.5);
\end{tikzpicture}   
		\caption{Worst case adversarial attack domain}
		\label{fig:adversarial train domain}
   \end{minipage}
 \flushleft
 Figure~\ref{fig:train domain} shows the domain for clean training points while the dashed cube in Figure~\ref{fig:adversarial train domain} shows the worst case adversarial attack domain (slightly bigger than the original training domain). Each new worst case adversarially attacked point is judiciously picked from within the green spheres around the corresponding clean training point with radius $\epsilon$ in a predefined norm.
	
\end{figure*} 

Before proceeding to the main discussion, we briefly discuss the notations and basic mathematical definitions used in the paper.

\paragraph{Notation:} We use a lowercase alphabet such as $x$ to denote a scalar, a lowercase bold alphabet such as $\mathbf{x}$ to denote a vector and an uppercase bold alphabet such as $\mathbf{X}$ to denote a matrix. The $i^{\text{th}}$ entry of the vector $\mathbf{x}$ is denoted by $\mathbf{x}_{i}$. The superscript star on a vector or matrix such as $\mathbf{x}^{\star}$ denotes it is the optimal solution for some optimization problem. A general norm for a vector is denoted by $\norm{\mathbf{x}}$, and its dual norm is indicated by a subscript  asterisk, such as $\norm{\mathbf{x}}_*$. The set $\{1,2,\ldots, n\}$ is denoted by $[n]$. 
A set is represented by capital calligraphic alphabet such as $\mathcal{P}$, and its cardinality is represented by  $|\mathcal{P}|$. For a scalar $x$, $|x|$ represents its absolute value.  
\begin{definition}
    \label{def:dualnorm}
    The dual norm of a vector, $ \|\cdot \|_{*}$  is defined as: 
\begin{align}
   \|\mathbf{z}\|_{*}=\sup\{\mathbf{z}^{\intercal }\mathbf{x}\mid \|\mathbf{x}\|\leq 1\} 
\end{align}
\end{definition}
\begin{definition}
	\label{def:subdiff}
Let $\norm{\cdot}_*$ is the dual norm to $\norm{.}$. The sub-differential of a norm is defined as:
\begin{align}
\partial \norm{\mathbf{x}} = \{ \mathbf{v} : \mathbf{v}^{\intercal }\mathbf{x}= \norm{\mathbf{x}}, \norm{\mathbf{v}}_* \leq 1 \}
\end{align}

\end{definition}
In this work, we propose \emph{plug-and-play} solutions for various ML problems to enable adversarially robust training. By plug-and-play solution, we mean that any addition to the existing algorithm comes in terms of a closed-form equation or as a solution to an easy-to-solve optimization problem. Such a solution can be integrated with the existing algorithm with very minimal changes.
\section{Warm Up}
\label{sec:warmup}
In this section, we formally discuss our proposed approach of adversarially robust training on some warm-up problems. The classical approach to estimate model parameters in various ML problems is to minimize a loss function using an optimization algorithm such as gradient descent \citep{ruder2016overview,chen2015fast, andrychowicz2016learning}.
Turning the focus to adversarially robust training, we work with the following optimization problem in supervised learning, which can be found in \citep{yin2019rademacher}:
\begin{align}
    \label{eq:adversarial opt prob}
    \hat{\mathbf{w}} = \argmin_{\mathbf{w}} \frac{1}{n} \sum_{i = 1}^n \sup_{\norm{\mathbf{\Delta}} \leq \epsilon} \;\; l(\mathbf{x}^{(i)} + \mathbf{\Delta}, y^{(i)}, \mathbf{w}) 
\end{align}
For unsupervised learning problems, one just removes the variables $y^{(i)}$ above.
The optimization problem~\ref{eq:adversarial opt prob} depends on two variables: the adversarial perturbation $\mathbf{\Delta}$ and the model parameter $\mathbf{w}$. We solve for one variable assuming the other is given iteratively as illustrated in Algorithm \ref{alg:main}.  Specifically, we estimate $\mathbf{\Delta}^{\star}$ for robust learning by defining the worst case adversarial attack  for a given parameter vector $\mathbf{w}^{(j-1)}$ ($j$ denotes the iteration number in gradient descent) and sample $\left\{ \mathbf{x}^{(i)}, y^{(i)}\right\}$  as follows:
\begin{align*}
    \mathbf{\Delta}^{\star} = \argsup_{\norm{\mathbf{\Delta}} \leq \epsilon} \;\; l(\mathbf{x}^{(i)} + \mathbf{\Delta}, y^{(i)}, \mathbf{w}^{(j-1)})
\end{align*}
For brevity, we drop the subscript $j-1$ from the parameter $\mathbf{w}$ when it is clear from the context that the optimization problem is being solved for a particular iteration.
\begin{algorithm}[htb]
\caption{Plug and play algorithm}
\label{alg:main}
\KwIn{$\left\{ \mathbf{x}^{(i)}, y^{(i)}\right\}$ for $i \in [n]$, $T$: number of iterations, $\eta_j:$ step size for iteration $j \in [T]$}
$\mathbf{w}^{(0)} \leftarrow $ initial value \; 
\For {$j=1$ \thinspace to T}{
gradient $\leftarrow$ 0 \;
\For {$i=1$ \thinspace to n}{
    $\mathbf{\Delta}^{\star} = \arg \sup_{\norm{\mathbf{\Delta}} \leq \epsilon}  l(\mathbf{x}^{(i)} + \mathbf{\Delta}, y^{(i)}, \mathbf{w}^{(j-1)}) $\;
    
    gradient $\leftarrow$ gradient + $\frac{\partial l(\mathbf{x}^{(i)} + \mathbf{\Delta}^{\star}, y^{(i)}, \mathbf{w}^{(j-1)})}{\partial \mathbf{w}}$\; 
}
$\mathbf{w}^{(j)} \leftarrow \mathbf{w}^{(j-1)} - \eta_j \frac{1}{n} $gradient
}
\KwOut{$\hat{\mathbf{w}} = \mathbf{w}^{(T)}$}
\end{algorithm}
Naturally, computing $\mathbf{\Delta}^{\star}$ by solving another maximization problem every time might not be necessarily efficient. To tackle this issue, we provide plug-and-play solutions of $\mathbf{\Delta}^{\star}$ for a given $(\mathbf{x}^{(i)}, y^{(i)}, \mathbf{w}^{(j)})$ where $i \in [n]$ and $j \in [T]$ for various widely-used ML problems. 
\subsection{Linear Regression}
\label{sec:regression} 
We start with a linear regression model which is used in various applications across numerous domains such as biology \citep{schneider2010linear}, econometrics, epidemiology, and finance \citep{myers1990classical}. 
As discussed in the previous sub-section, the adversary tries to perturb each sample to the maximum possible extent using the budget $\epsilon$ by solving the following maximization problem for each sample:
\begin{align}
\mathbf{\Delta}^{\star} = \argsup_{\lmod \mathbf{\Delta} \rmod \leq \epsilon} \left( \mathbf{w}^{\intercal} \left( \mathbf{x}^{(i)} + \mathbf{\Delta} \right) - y^{(i)}\right)^2 \label{eq:l2loss_lpconst}
\end{align}
where $y^{(i)} \in \mathbb{R}$, $\mathbf{x}^{(i)}, \mathbf{\Delta} \in \mathbb{R}^d$ and $\lmod \mathbf{\Delta} \rmod$ denotes any general norm. We provide the following theorem to compute $\mathbf{\Delta}^{\star}$ in closed form.
\begin{theorem}
\label{thm:reg_sqerr}
For any general norm $\|\cdot\|$, the solution for problem in Eq. \eqref{eq:l2loss_lpconst} for a given $\left(\mathbf{x}^{(i)}, y^{(i)} \right)$ is:
\begin{align*}
\mathbf{\Delta}^{\star} =  \begin{cases}
\pm \epsilon \frac{\mathbf{v}}{\norm{\mathbf{v}}}, \quad  & \qquad \mathbf{w}^{\intercal}\mathbf{x}^{(i)} - y^{(i)} = 0 \\
 \sign(\mathbf{w}^{\intercal}\mathbf{x}^{(i)} - y^{(i)} ) \epsilon \frac{\mathbf{v}}{\norm{\mathbf{v}}}  & \qquad \mathbf{w}^{\intercal}\mathbf{x}^{(i)} - y^{(i)} \neq 0 
\end{cases}
\end{align*}
where $\mathbf{v} \in \partial\norm{\mathbf{w}}_* $ as specified in Definition \ref{def:subdiff}. 
\end{theorem}
\subsection{Logistic Regression}
\label{sec:logit} 
Next, we tackle logistic regression which is widely used for classification tasks in many fields such as medical diagnosis~\citep{truett1967multivariate}, marketing~\citep{michael1997data} and biology~\citep{freedman2009statistical}. Using previously introduced notations,  we formulate logistic regression \citep{kleinbaum2002logistic} with worst case adversarial attack in the following way: 
{\small
\begin{align}
\mathbf{\Delta}^{\star} = \argsup_{\lmod \mathbf{\Delta} \rmod \leq \epsilon} \;\; \log\left( 1 + \exp\left( -y^{(i)} \mathbf{w}^{\intercal} \left(\mathbf{x}^{(i)} + \mathbf{\Delta}\right)\right)\right) \label{eq:logit_lpconst}
\end{align}}
where $y^{(i)} \in \left\{-1,1 \right\}$ and $\mathbf{x}^{(i)}, \mathbf{\Delta} \in \mathbb{R}^d$. The optimal solution for above optimization problem is provided in the following theorem. 
\begin{theorem}
	\label{thm:logit}
	For any general norm $\|\cdot\|$, and the problem specified in Eq. \eqref{eq:logit_lpconst}, the optimal solution is given by $\mathbf{\Delta}^{\star} = - \epsilon y^{(i)} \nicefrac{\mathbf{v}}{\norm{\mathbf{v}}}$, where $\mathbf{v} \in \partial\norm{\mathbf{w}}_*  $ as specified in Definition \ref{def:subdiff}.
\end{theorem}
Theorem \ref{thm:hinge} presented in Section \ref{sec:hinge} of appendix discusses a similar result for the hinge loss. 
\section{Main Results}
\label{sec:main}
\subsection{Two-Layer Neural Networks} \label{sec:twolayernn}

Consider a two-layer neural network for a binary classification problem with any general (convex or nonconvex) activation function $\sigma: \mathbb{R} \rightarrow \mathbb{R}$ in the first layer. As we work on the classification problem, we consider the log sigmoid activation function in the final layer. The general adversarial problem can be stated as: 
{\small
\begin{align}
    \mathbf{\Delta}^{\star} = \argsup_{\lmod \mathbf{\Delta} \rmod \leq \epsilon} \;\; \log \nrbr{ 1 + \exp\nrbr{-y^{(i)} \mathbf{v}^{\intercal}\text{$\boldsymbol{\sigma}$} \nrbr{\mathbf{W}^{\intercal} \nrbr{\mathbf{x}^{(i)} + \mathbf{\Delta}}}}} \label{eq:2layer_gen1}
\end{align}}
where $y^{(i)} \in \crbr{1, -1}$ for binary classification, $\mathbf{x}, \mathbf{\Delta} \in \mathbf{R}^d$, and the weight parameters $\mathbf{W} \in \mathbb{R}^{h \times d}$, and $\mathbf{v} \in \mathbb{R}^h$. Note that $h$ denotes the number of hidden units in the first layer, and the output for the general activation function $\text{$\boldsymbol{\sigma}$}: \mathbb{R}^h \rightarrow \mathbb{R}^h$ is obtained by applying $\sigma: \mathbb{R} \rightarrow \mathbb{R}$ to each dimension independently. The optimal solution to the above problem is the following theorem.

\begin{theorem}
    \label{thm:2nns}
    For any general norm $\|\cdot\|$, and any activation function, the optimal solution for the problem specified in Eq. \eqref{eq:2layer_gen1} can be solved using difference of convex functions. 
\end{theorem}
\begin{proof}
As $\log(\cdot)$ and $\exp(\cdot)$ are monotonically increasing functions, the adversarial problem specified in Eq. \eqref{eq:2layer_gen1} can be equivalently expressed as: 
\begin{align}
    \mathbf{\Delta}^{\star} = \argmin_{\lmod \mathbf{\Delta} \rmod \leq \epsilon} \;\; f(\mathbf{\Delta}) = y \mathbf{v}^{\intercal}\text{$\boldsymbol{\sigma}$} \nrbr{\mathbf{W}^{\intercal} \nrbr{\mathbf{x} + \mathbf{\Delta}}} \label{eq:delta_2nn}
\end{align}
where we have dropped the subscript (i) for brevity, as it is clear that the above problem is solved for each sample. The objective function of the above problem can be equivalently represented as: 
\begin{align}
    f(\mathbf{\Delta}) = \sum_{i: y\mathbf{v}_i > 0} y\mathbf{v}_i \sigma\nrbr{\mathbf{z}_i} - \sum_{i: y\mathbf{v}_i < 0}  |y\mathbf{v}_i| \sigma\nrbr{\mathbf{z}_i} \label{eq:NN_obj1},\qquad \mathbf{z}_i = \mathbf{W}_i^{\intercal} \nrbr{\mathbf{x} + \mathbf{\Delta}}
\end{align}
where $\mathbf{W}_i$ represents the $i^{\text{th}}$ row of matrix $\mathbf{W}$. Further, we express any general activation function $\sigma(\cdot)$ (which may be non-convex) as the difference of two convex functions using $\sigma(\mathbf{\Delta}) = \sigma_1(\mathbf{\Delta}) - \sigma_2(\mathbf{\Delta}) $. Using this formulation, the objective function in Eq. \eqref{eq:NN_obj1} can be expressed as $f(\mathbf{\Delta}) = g(\mathbf{\Delta}) - h(\mathbf{\Delta})$, where $g(\mathbf{\Delta})$ and $h(\mathbf{\Delta})$ are convex functions defined as: 
\begin{align}
    g(\mathbf{\Delta}) & = \sum_{i: y\mathbf{v}_i > 0}  y\mathbf{v}_i \sigma_1(\mathbf{z}_i)  + \sum_{i: y\mathbf{v}_i < 0}  |y\mathbf{v}_i| \sigma_2(\mathbf{z}_i) \label{eq:g_def}\\
    h(\mathbf{\Delta}) &= \sum_{i: y\mathbf{v}_i > 0}  y\mathbf{v}_i \sigma_2(\mathbf{z}_i) + \sum_{i: y\mathbf{v}_i < 0}  |y\mathbf{v}_i| \sigma_1(\mathbf{z}_i) \label{eq:h_def}
\end{align}
where $\mathbf{z}_i$ is defined in Eq. \eqref{eq:NN_obj1}. It should be noted that $g(\mathbf{\Delta})$ and $h(\mathbf{\Delta})$ are convex functions as they are positive weighted combination of convex functions $\sigma_1(\cdot)$ and $\sigma_2(\cdot)$. As the objective function $f(\mathbf{\Delta})$ can be expressed as difference of convex functions for any activation function specified in Appendix \ref{sec:diff_conv}, we can use difference of convex functions algorithms  (DCA) \citep{tao1997convex}. 
\end{proof}
 If set $\mathcal{S} = \crbr{i \; | \; y\mathbf{v}_i < 0, i \in [h]} = \emptyset$ and we have an activation function $\sigma(\mathbf{\Delta})$ such that $\sigma_2(\mathbf{\Delta}) = 0$, then $h(\mathbf{\Delta}) = 0$ and the problem in Eq. \eqref{eq:2layer_gen1} reduces to a convex optimization problem. This may not be the case in general for two-layer neural networks. Therefore we use the difference of convex programming approach \citep{tao1997convex, sriperumbudur2009convergence, yu2021convergence, abbaszadehpeivasti2021rate, le2009convergence, yen2012convergence, khamaru2018convergence, nitanda2017stochastic} which are proved to converge to a critical point. 

The first step to solve this optimization problem is constructing the functions $g(\mathbf{\Delta})$ and $h(\mathbf{\Delta})$, which requires decomposing the activation functions as the difference of convex functions.  In order to do this, we decompose various activation functions commonly used in the literature as a difference of two convex functions. The decomposition is done by constructing a linear approximation of the activation function around the point where it changes the curvature. (Please see Appendix \ref{sec:diff_conv}). 

Further, we compute $\mathbf{\Delta}^{\star}$ defined in Eq. \eqref{eq:delta_2nn} by expressing $f(\mathbf{\Delta}) = g(\mathbf{\Delta}) - h(\mathbf{\Delta}) $ and using concave-convex procedure \citep{sriperumbudur2009convergence} or difference of convex function algorithm (DCA)  \citep{tao1997convex}. These algorithms are established to converge to a critical point, and hence the obtained $\mathbf{\Delta}^{\star}$ is plugged in Algorithm \ref{alg:main}.
\subsection{Learning Gaussian Graphical Models}
\label{sec:graphical}
Next, we provide a robust adversarial training process for learning Gaussian graphical models. These models are used to study the conditional independence of jointly Gaussian continuous random variables. This can be analyzed by inspecting the zero entries in the inverse covariance matrix, popularly referred as the precision matrix and denoted by $\mathbf{\Omega}$ \citep{lauritzen1996graphical,honorio2012variable}. The classical (non-adversarial) approach \citep{yuan2007model} solves the following optimization problem to estimate $\mathbf{\Omega}$ : 
\begin{align*}
    \mathbf{\Omega}^{\star} = \argmin_{\mathbf{\Omega} \succ 0} -\log(\det(\mathbf{\Omega})) + \frac{1}{n}\sum_{i = 1}^n \mathbf{x}^{(i) \intercal}\mathbf{\Omega}\mathbf{x}^{(i)} + c \norm{\mathbf{\Omega}}_1
\end{align*}
where $\mathbf{\Omega} $ is constrained to be a symmetric positive definite matrix and $c$ is a positive regularization constant. As the first term $\log(\det(\mathbf{\Omega})) $ in the above equation can not be influenced by adversarial perturbation in $\mathbf{x}^{(i)}$, we define the adversarial attack problem for this case as maximizing the second term by perturbing $\mathbf{x}^{(i)}$ for each sample:  
\begin{align}
\mathbf{\Delta}^{\star} = \argsup_{\lmod \mathbf{\Delta} \rmod \leq \epsilon} \left( \mathbf{x}^{(i)} + \mathbf{\Delta} \right)^{\intercal} \mathbf{\Omega} \left( \mathbf{x}^{(i)} + \mathbf{\Delta} \right)
\label{eq:Graph_lpconst1}
\end{align}

For the above problem, we provide solutions for the $\ell_2$ and $\ell_\infty$ norm constraints as follows. 
\begin{theorem}
	\label{thm:Gaussmdl_l2}
	The solution for the problem in Eq. \eqref{eq:Graph_lpconst1} with $\ell_2$ constraint on $\mathbf{\Delta}$ is given by	$\mathbf{\Delta}^{\star}= \left(\mu^{\star} \mathbf{I} - \mathbf{\Omega}\right)^{-1}\mathbf{\Omega \mathbf{x}}^{(i)}$, where $\mu^{\star}$ can be derived from the following 1-D optimization problem: 
\begin{align}
\max \quad -\frac{1}{2}\mathbf{x}^{{(i)} \intercal} \mathbf{\Omega} \left(\mu \mathbf{I} - \mathbf{\Omega}\right)^{-1}\mathbf{\Omega \mathbf{x}}^{(i)}- \frac{\mu \epsilon^2}{2}, \qquad  \text{such that}  \quad  \mu \mathbf{I} - \mathbf{\Omega} \succeq 0  \label{eq:graph_muopt}
\end{align}
\end{theorem}

\begin{theorem}
    \label{thm:Gaussmdl_linf}
The solution for the problem specified in Eq. \eqref{eq:Graph_lpconst1} with $\ell_{\infty}$ constraint on $\mathbf{\Delta} \in \mathbb{R}^{p}$ can be derived from the last column/row of  $\mathbf{Y}$ obtained from the following optimization problem: 
\begin{align*}
    \max \; \left\langle \begin{bmatrix}
\mathbf{\Omega} & \mathbf{\Omega }\mathbf{x}^{(i)} \\
\left( \mathbf{\Omega} \mathbf{x}^{(i)}\right)^{\intercal} & 0
\end{bmatrix}, \mathbf{Y} \right\rangle \qquad 
\text{such that} \;\; \mathbf{Y}_{p+1, p+1} = 1, \;  \mathbf{Y} \succeq 0, \; |\mathbf{Y}_{ij}| \leq \epsilon^2 \; \forall i,j \in [p]
\end{align*}
\end{theorem}
The results in Theorem~\ref{thm:Gaussmdl_l2} and Theorem~\ref{thm:Gaussmdl_linf} do not have a closed form but can be computed easily by solving a standard one-dimensional optimization problem and a SDP respectively. Very efficient scalable SDP solvers exist in practice \citep{yurtsever2021scalable}. 
\subsection{Matrix Completion}
\label{sec:matcompl}

Assume we are given a partially observed matrix $\mathbf{X}$. Let  $\mathcal{P}$ be a set of indices where the entries of $\mathbf{X}$ are observed (i.e., not missing). The classical (non-adversarial) matrix completion approach aims to find a low rank matrix \citep{shamir2011collaborative} with small squared error in the observed entries: 
\begin{align*}
    \min_{\mathbf{Y}} \quad & \sum_{(i,j) \in \mathcal{P}}\left(\mathbf{X}_{ij}  - \mathbf{Y}_{ij} \right)^2 + c \| \mathbf{Y} \|_{\text{tr}} 
\end{align*}
where $c$ is a positive regularization constant and $\|\cdot \|_{\text{tr}}$ denotes trace norm of matrix which ensures low-rankness. Note that regularization does not impact the adversarial training framework. We define the following worst-case adversarial attack problem: 
\begin{align}
\mathbf{\Delta}^{\star} = \argsup_{\lmod \mathbf{\Delta} \rmod \leq \epsilon}  \sum_{(i,j) \in \mathcal{P}}\left(\mathbf{X}_{ij} + \mathbf{\Delta}_{ij} - \mathbf{Y}_{ij} \right)^2 \label{eq:mat_compl_1}
\end{align}
The solution for the above problem for the Frobenius norm constraint and entry-wise $\ell_{\infty}$ constraint on  $\mathbf{\Delta}$ is proposed in Theorem \ref{thm:mat_compl_fro} and Corollary \ref{cor:matrixcompl_linf}. 
\begin{theorem}
	\label{thm:mat_compl_fro}
	The optimal solution for the optimization in Eq. \eqref{eq:mat_compl_1} with Frobenius norm constraint on  $\mathbf{\Delta}$ if $\exists (i,j) \in \mathcal{P}$ such that $\mathbf{X}_{ij} \neq \mathbf{Y}_{ij}$ is given by 
 {\small
	\begin{align*}
	\mathbf{\Delta}^{\star}_{ij} = \begin{cases}
	\epsilon \frac{\left(\mathbf{X}_{ij} - \mathbf{Y}_{ij} \right) }{\sqrt{\sum\limits_{(i,j) \in \mathcal{P}} \left(\mathbf{X}_{ij} - \mathbf{Y}_{ij} \right)^2}} & \qquad (i,j) \in \mathcal{P} \\
	0 & \qquad (i,j) \notin \mathcal{P}
	\end{cases}
	\end{align*}}
	If $\mathbf{X}_{ij} = \mathbf{Y}_{ij}, \forall (i,j) \in \mathcal{P}$, then the optimal $\mathbf{\Delta}^{\star}$ can be any solution satisfying $\sum_{(i,j) \in \mathcal{P}} \mathbf{\Delta}^2_{ij} = \epsilon$.
\end{theorem}
\begin{corollary}
\label{cor:matrixcompl_linf}
	The optimal solution for the optimization problem in Eq. \eqref{eq:mat_compl_1} with the constraint $|\mathbf{\Delta}_{ij}| \leq \epsilon$ for all  $(i,j) \in \mathcal{P}$ is given by $\mathbf{\Delta}_{ij} = \frac{(\mathbf{X}_{ij} - \mathbf{Y}_{ij}) }{|\mathbf{X}_{ij} - \mathbf{Y}_{ij}| } \epsilon$.  
 \end{corollary}

\subsection{Max-Margin Matrix Completion}
\label{sec:maxmargin_matcompl}

We start the discussion from the problem under the classical (non-adversarial) setting. Consider a partially observed label matrix where the observed entries are $+1$ or $-1$. Let $\mathcal{P}$ be the indices of the observed entries. The problem of max-margin matrix completion \citep{srebro2004maximum} is defined as follows:
\begin{align}
  \min_{\mathbf{Y}} \;\;   \sum_{(i,j) \in \mathcal{P}} \max( 0, 1- \mathbf{X}_{ij} \mathbf{Y}_{ij} ) + c \lmod \mathbf{Y} \rmod_{\text{tr}}  \label{eq:lossfn_maxmargin}
\end{align}
where $c>0$ is a regularization constant and $\lmod \cdot \rmod_{\text{tr}}$ represents the trace norm \citep{bach2008consistency}. As the second term, $\lmod \mathbf{Y} \rmod_{\text{tr}}$ in the above optimization problem can not be affected by the adversary, we define the worst-case adversarial attack problem as the maximization of the first term with $\epsilon$ radius around $\mathbf{X}$:
{\small
\begin{align}
\mathbf{\Delta}^{\star} = \argsup_{\lmod \mathbf{\Delta} \rmod \leq \epsilon}  \sum_{(i,j) \in \mathcal{P}} \max( 0, 1- \left(\mathbf{X}_{ij} + \mathbf{\Delta}_{ij} \right) \mathbf{Y}_{ij} ) \label{eq:maxmarginmat_prbdef}
\end{align}}
The optimal $\mathbf{\Delta}^{\star}$ under Frobenius norm constraint on $\mathbf{\Delta}$ is proposed in the following theorem.
\begin{theorem}
	\label{thm:maxmarginmat_fro}
	For the problem in Eq.  \eqref{eq:maxmarginmat_prbdef} with Frobenius norm constraint on $\mathbf{\Delta}$, the solution is:
 {\footnotesize
	\begin{align*}
	\mathbf{\Delta}_{ij} = \begin{cases}
	-\mathbf{Y}_{ij} \frac{\epsilon}{\sqrt{\sum\limits_{(i,j) \in \mathcal{P} } \mathbf{Y}^2_{ij} }}  \qquad & (i,j) \in \mathcal{P}_1 \\ 
\quad	0 \qquad &  (i,j) \notin \mathcal{P}_1
	\end{cases}
	\end{align*}}
	where $\mathcal{P}_1 \subseteq \mathcal{P}$ is chosen by sorting $\mathbf{X}_{ij} \mathbf{Y}_{ij} $ and selecting indices which satisfy $\mathbf{X}_{ij} \mathbf{Y}_{ij} < 1 + \nicefrac{\epsilon}{\sqrt{\sum\limits_{(i,j) \in \mathcal{P} } \mathbf{Y}^2_{ij} }}$.
\end{theorem}
Similarly, the solution for the problem in Eq. \eqref{eq:maxmarginmat_prbdef} for the entry-wise $\ell_{\infty}$ norm is proposed as follows. 
\begin{corollary}
	\label{cor:maxmarginmat_linf}
	For the problem in Eq.  \eqref{eq:maxmarginmat_prbdef} with the constraint $|\mathbf{\Delta}_{ij}| \leq \epsilon$ for all  $(i,j) \in \mathcal{P}$, the optimal solution is given by $\mathbf{\Delta}_{ij} = -\text{sign}(\mathbf{Y}_{ij}) \epsilon$.
\end{corollary}
\section{Implications}
\subsection{Novel Rademacher complexities}
\label{sec:rademacher}

The motivation behind deriving the bounds for the adversarial Rademacher complexity is briefly discussed here. As shown in Theorem 8 in \citep{bartlett2002rademacher}, the upper bound for generalization (the gap between empirical risk and population/expected risk) is $\hat{\mathfrak{R}}_S(\mathcal{F}) + \mathcal{O}(1/\sqrt{n})$. Hence an upper bound of order $\mathcal{O}(1/\sqrt{n})$ for the adversarial Rademacher complexity allows obtaining a $ \mathcal{O}(1/\sqrt{n})$ generalization bound. Also, as shown in Proposition 4.12 in \citep{wainwright2019high}, the lower bound for generalization is $\frac{1}{2}\hat{\mathfrak{R}}_S(\mathcal{F}) - \Omega(1/\sqrt{n})$. Hence a lower bound of order $\Omega(1/\sqrt{n})$ for the adversarial Rademacher complexity allows obtaining a $ \Omega(1/\sqrt{n})$ generalization bound.

\begin{table*}
	\caption{A summary of our results for Rademacher complexity of various loss functions and norm constraints.
    $\epsilon$ is the radius of the norm constraint, $n$ is the number of training samples for regression and classification, $|\mathcal{P}|$ is the number of observed entries in matrix completion.
    The empirical Rademacher complexity is formally described in Definition \ref{def:emp_RC} for regression and classification, and in Definition \ref{def:RC_2d} for matrix completion.
    Lipschitz losses include the logistic loss, the hinge loss, the hyperbolic tangent, the logit, among others.}
	\label{tab:all_results_RC}
	{\footnotesize
	\begin{tabular}{@{}l@{\hspace{0.1in}}p{0.6in}@{\hspace{0.1in}}p{0.6in}@{\hspace{0.1in}}p{0.6in}@{\hspace{0.1in}}p{0.6in}@{\hspace{0.1in}}l@{\hspace{0.1in}}l@{\hspace{0.1in}}l@{}}
		\toprule 
		&Problem &  Loss & Norm & Prior &Our result & Lower bound of $\hat{\mathfrak{R}}_S(\Tilde{\mathcal{F}})$ & Upper bound of $\hat{\mathfrak{R}}_S(\Tilde{\mathcal{F}})$ \\
		& & function & constraint & results & & & \\
		\midrule
		\multirow{4}{*}{\begin{turn}{90}  Main results\end{turn}} &Regression & Squared loss & Any norm &  None &  Theorem \ref{thm:RC_linreg} & $\hat{\mathfrak{R}}_S(\mathcal{F}) - \mathcal{O}(\epsilon/n + \epsilon^2/\sqrt{n})$ & $\hat{\mathfrak{R}}_S(\mathcal{F}) + \mathcal{O}(\epsilon/n + \epsilon^2/\sqrt{n})$ \\
		&Classification &  Lipschitz loss & Any norm & $\ell_{\infty}$ norm \citep{yin2019rademacher} & Theorem \ref{thm:RC_linclass} & $\max(\hat{\mathfrak{R}}_S(\mathcal{F}), \mathcal{O}(\epsilon/\sqrt{n}))$ & $\hat{\mathfrak{R}}_S(\mathcal{F}) + \mathcal{O}(\epsilon/\sqrt{n})$ \\
		& & Lipschitz loss & Any norm & $\ell_{p}$ norm \citep{awasthi2020adversarial} &  \\
        &Matrix Completion & Squared loss  & Entry-wise $\ell_{\infty}$ & None & Theorem \ref{thm:RC_MC} & $\hat{\mathfrak{R}}_S(\mathcal{F}) - \mathcal{O}(\epsilon/|\mathcal{P}|)$ & $\hat{\mathfrak{R}}_S(\mathcal{F}) + \mathcal{O}(\epsilon/|\mathcal{P}|)$ \\
        &Max-Margin MC & Lipschitz loss  & Entry-wise $\ell_{\infty}$& None & Theorem \ref{thm:RC_mmmc} & $\hat{\mathfrak{R}}_S(\mathcal{F}) - \mathcal{O}(\epsilon/|\mathcal{P}|)$ & $\hat{\mathfrak{R}}_S(\mathcal{F}) + \mathcal{O}(\epsilon/|\mathcal{P}|)$ \\ 
        & & Lipschitz loss  & Entry-wise $\ell_{\infty}$& None & Corollary \ref{corr:RC_mmmc_lbtight} & $\max(\hat{\mathfrak{R}}_S(\mathcal{F}), \mathcal{O}(\epsilon/|\mathcal{P}|))$ & $\hat{\mathfrak{R}}_S(\mathcal{F}) + \mathcal{O}(\epsilon/|\mathcal{P}|)$ \\
		\bottomrule
	\end{tabular}}
\end{table*}   

Before diving into theorems and their proofs, we summarize our theoretical statistical/generalization contribution on Rademacher complexity bounds in Table \ref{tab:all_results_RC}, and make some important observations:
\begin{itemize}
    \item The bounds for adversarial Rademacher complexity presented in Theorem \ref{thm:RC_linreg} for linear regression (for any general norm), Theorem \ref{thm:RC_MC} for matrix completion, Theorem \ref{thm:RC_mmmc} and Corollary \ref{corr:RC_mmmc_lbtight} for max-margin matrix completion are novel. 
    \item The bounds presented in Theorem \ref{thm:RC_linclass} for linear classifiers are for any general norm can be seen as a generalization of the $\ell_{\infty}$ norm \citep{yin2019rademacher} or any particular $p$-norm \citep{awasthi2020adversarial}. 
\end{itemize}
We utilize the closed-form solutions for various adversarial problems presented in Section \ref{sec:warmup} and Section \ref{sec:main} to derive new bounds for the adversarial Rademacher complexity. To start with, we define the empirical Rademacher complexity formally. 
\begin{definition}[Rademacher complexity for regression and classification]
    \label{def:emp_RC}
    The empirical Rademacher complexity of the hypothesis class $\mathcal{F}$ with respect to a data set $S = \crbr{\nrbr{\mathbf{x}^{(1)}, y^{(1)}}, \ldots, \nrbr{\mathbf{x}^{(n)}, y^{(n)}}}$ is defined as: 
    {\small
    \begin{align*}
        \hat{\mathfrak{R}}_S(\mathcal{F}) = \frac{1}{n} \E_{\sigma} \sqbr{\sup_{h\in \mathcal{F}}\nrbr{  \sum_{i = 1}^n \sigma_i h \nrbr{\mathbf{x}^{(i)}, y^{(i)}} }}
    \end{align*}}
\end{definition}
We denote the adversarial function class with $\Tilde{\mathcal{F}}$. Further, we utilize Theorem \ref{thm:reg_sqerr} to derive upper and lower bounds of the adversarial Rademacher complexity for linear regression. 
\begin{theorem}[Regression with squared loss] 
    \label{thm:RC_linreg}
    Let the function class $\mathcal{F} = \{ \nrbr{\mathbf{x}^{(i)}, y^{(i)}} \mapsto \left( \mathbf{w}^{\intercal} \mathbf{x}^{(i)}- y^{(i)}\right)^2\;|\; \norm{\mathbf{w}}_{*} \leq C \}$ and the adversarial function class $\Tilde{\mathcal{F}} = \{ \nrbr{\mathbf{x}^{(i)}, y^{(i)}} \mapsto \sup\limits_{\norm{\mathbf{\Delta}} \leq \epsilon}\left( \mathbf{w}^{\intercal} ( \mathbf{x}^{(i)} + \mathbf{\Delta} ) - y^{(i)}\right)^2\;|\; \norm{\mathbf{w}}_{*} \leq C \}$, then:
    {\small
    \begin{align*}
\hat{\mathfrak{R}}_S(\mathcal{F}) - b(n,\epsilon, C)   &\leq \hat{\mathfrak{R}}_S(\Tilde{\mathcal{F}}) \leq  \hat{\mathfrak{R}}_S(\mathcal{F}) + b(n,\epsilon, C), \; \; \text{where} \; b(n,\epsilon, C) = \frac{2\epsilon C}{n} \E_{\sigma} \sqbr{\norm{\sum_{i = 1}^n  \sigma_i \mathbf{x}^{(i)}  }   } + \frac{\epsilon^2 C^2}{\sqrt{n}} 
\end{align*}}
\end{theorem}
It should be noted that the above theorem is proposed for any general norm and any further simplification of the term $\E_{\sigma} \sqbr{\norm{ \nicefrac{1}{n} \sum_{i = 1}^n  \sigma_i \mathbf{x}^{(i)}  }   } $ requires a further assumption on the specific norm. For example, it can be easily bounded as $\nicefrac{c_1}{\sqrt{n}}$ using Khintchine's inequality for the Euclidean norm, where $c_1$ is some constant.

\begin{proof}[Proof sketch]
To prove the above theorem, we have used carefully manipulated the Rademacher complexity of adversarial function class in terms of classical function class with the help of Theorem \ref{thm:reg_sqerr}. Further, we have used the Ledoux-Talagrand contraction principle for Lipchitz functions and Khintchine's inequality. A similar proof recipe is used to prove other theorems in this section, whose detailed proofs can be seen in Section \ref{sec:appn_RC_proof} of supplementary material.
\end{proof}

Further, we utilize Theorem \ref{thm:logit} or Theorem \ref{thm:hinge} to derive the upper and lower bounds of the adversarial Rademacher complexity for linear classification. 
Similar to the existing works \citep{yin2019rademacher, awasthi2020adversarial} mentioned in Table \ref{tab:all_results_RC} as well as prior work on non-adversarial regimes \citep{KakadeST08}, we analyze the Rademacher complexity of linear functions, which allows for bounding the Rademacher complexity of relatively more complex Lipschitz functions (e.g., logistic loss, hinge loss, hyperbolic tangent, logit) by using the Ledoux-Talagrand contraction lemma \citep{ledoux2013probability}.

\begin{theorem}[Classification]
        \label{thm:RC_linclass}
Let the function class $\mathcal{F} = \{\nrbr{\mathbf{x}^{(i)}, y^{(i)}} \mapsto -y^{(i)}\mathbf{w}^{\intercal}\mathbf{x}^{(i)}\;|\; \norm{\mathbf{w}}_{*} \leq C \}$ and the adversarial function class $\Tilde{\mathcal{F}} = \{\nrbr{\mathbf{x}^{(i)}, y^{(i)}} \mapsto \sup_{\norm{\mathbf{\Delta}} \leq \epsilon} -y^{(i)}\mathbf{w}^{\intercal}(\mathbf{x}^{(i)}+\mathbf{\Delta})\;|\; \norm{\mathbf{w}}_{*} \leq C \}$, then:
\begin{align*}
  \max\crbr{\hat{\mathfrak{R}}_S(\mathcal{F}), \nicefrac{C\epsilon}{2\sqrt{2n}} } \leq\hat{\mathfrak{R}}_S(\Tilde{\mathcal{F}}) \leq  \hat{\mathfrak{R}}_S(\mathcal{F}) + \nicefrac{C\epsilon}{\sqrt{n}}
\end{align*}
\end{theorem}
Similar results have been proposed in the literature for the $\ell_{\infty}$ norm \citep{yin2019rademacher} or any particular $p$-norm \citep{awasthi2020adversarial}, whereas our result is more general, pertaining to any norm.

We further move on to the matrix completion problem and propose the following definition of the Rademacher complexity motivated from Definition \ref{def:emp_RC}. 
\begin{definition}[Rademacher complexity for matrix completion]
    \label{def:RC_2d}
     The empirical Rademacher complexity of the hypothesis class $\mathcal{F}$ with respect matrix $\mathbf{X}$ with observed entries $(i,j) \in \mathcal{P}$ is defined as: 
     {\footnotesize
    \begin{align*}
        \hat{\mathfrak{R}}_{\mathbf{X}}(\mathcal{F}) = \frac{1}{n} \E_{\sigma} \sqbr{\sup_{h\in \mathcal{F}}\nrbr{  \sum_{(i,j) \in \mathcal{P}} \sigma_{ij} h ( \mathbf{X}_{ij} ) }}
    \end{align*}}
\end{definition}
We further utilize Corollary \ref{cor:matrixcompl_linf} to derive novel upper and lower bounds for the adversarial Rademacher complexity for matrix completion. 
\begin{theorem}[Matrix completion with squared loss]
    \label{thm:RC_MC}
    Let the function class $\mathcal{F} = \{\mathbf{X}_{ij} \mapsto \nrbr{\mathbf{X}_{ij} - \mathbf{Y}_{ij}}^2\;|\; \norm{\mathbf{Y}}_* \leq C \}$ and the adversarial function class $\Tilde{\mathcal{F}} = \{\mathbf{X}_{ij} \mapsto \sup\limits_{\norm{\mathbf{\Delta}_{ij}} \leq \epsilon}  \nrbr{\mathbf{X}_{ij} + \mathbf{\Delta}_{ij} - \mathbf{Y}_{ij} }^2 \;|\; \norm{\mathbf{Y}}_* \leq C  \}$, then:
    \begin{align*}
 \hat{\mathfrak{R}}_{\mathbf{X}}(\mathcal{F}) -  \frac{2\epsilon C}{|\mathcal{P}|} \E  \norm{\bm{\sigma}} \leq   \hat{\mathfrak{R}}_{\mathbf{X}}(\Tilde{\mathcal{F}}) \leq  \hat{\mathfrak{R}}_{\mathbf{X}}(\mathcal{F}) + \frac{2\epsilon C}{|\mathcal{P}|} \E  \norm{\bm{\sigma}}
    \end{align*}
    where $\bm{\sigma}$ is a matrix whose entries $\bm{\sigma}_{ij}$ for $(i,j) \in \mathcal{P}$ follows Rademacher distribution.
\end{theorem}
The term $\E \norm{\bm{\sigma}}$ can be simplified for particular norms. For example, if we consider $\norm{\mathbf{Y}}_1 \leq C$, then $\E_{\sigma}  \sqbr{ \norm{\bm{\sigma}}_{\infty}} = 1$ where $\norm{.}_1$ and  $\norm{.}_{\infty}$ denote the entrywise $\ell_1$  and $\ell_{\infty}$ norm respectively. 

Further, we utilize Corollary \ref{cor:maxmarginmat_linf} to derive upper and lower bounds for the adversarial Rademacher complexity for max-margin matrix completion.
As previously discussed before Theorem \ref{thm:RC_linclass}, prior work in classification on adversarial regimes \citep{yin2019rademacher, awasthi2020adversarial} mentioned in Table \ref{tab:all_results_RC} as well as prior work on non-adversarial regimes \citep{KakadeST08} use the Rademacher complexity of linear functions to bound the Rademacher complexity of relatively more complex Lipschitz functions by using the Ledoux-Talagrand contraction lemma \citep{ledoux2013probability}.
We use the same principle for max-margin matrix completion which typically uses the hinge loss, which is Lipschitz.

\begin{theorem}[Max-margin matrix completion]
    \label{thm:RC_mmmc}
     Let the function class $\mathcal{F} = \{ \mathbf{X}_{ij} \mapsto - \mathbf{X}_{ij}\mathbf{Y}_{ij}\;|\; \norm{\mathbf{Y}}_* \leq C \}$ and the adversarial function class $\Tilde{\mathcal{F}} = \{ \mathbf{X}_{ij} \mapsto  \sup\limits_{\norm{\mathbf{\Delta}_{ij}} \leq \epsilon} - \nrbr{\mathbf{X}_{ij} + \mathbf{\Delta}_{ij} } \mathbf{Y}_{ij} \;|\; \norm{\mathbf{Y}}_* \leq C  \}$, then:
    \begin{align*}
      \hat{\mathfrak{R}}_{\mathbf{X}}(\mathcal{F}) - \frac{\epsilon C}{|\mathcal{P}|} \E  \norm{\bm{\sigma}}  \leq \hat{\mathfrak{R}}_{\mathbf{X}}(\Tilde{\mathcal{F}}) \leq \hat{\mathfrak{R}}_{\mathbf{X}}(\mathcal{F}) + \frac{\epsilon C}{|\mathcal{P}|} \E  \norm{\bm{\sigma}}
    \end{align*}
    where $\bm{\sigma}$ is a matrix whose entries $\bm{\sigma}_{ij}$ for $(i,j) \in \mathcal{P}$ follows Rademacher distribution.
\end{theorem}
\begin{corollary}[Max-margin matrix completion]
\label{corr:RC_mmmc_lbtight}
For the function class defined in Theorem \ref{thm:RC_mmmc} with $\ell_1$ norm constraint, i.e., $\norm{\mathbf{Y}}_1 \leq C $, we obtain the following tighter lower bound: $\max \crbr{ \hat{\mathfrak{R}}_{\mathbf{X}}(\mathcal{F}) , \nicefrac{\epsilon C}{|\mathcal{P}|} } \leq \hat{\mathfrak{R}}_{\mathbf{X}}(\Tilde{\mathcal{F}})$.
\end{corollary}

\begin{table*}
	\caption{Error metrics on real-world data sets for various supervised and unsupervised ML problems. Notice that the proposed approach performs better (most of the time) as compared to other methods (``No error'', ``Random'', ``FGSM'', ``PGD'', and ``TRADES''). MC and NN denote matrix completion and neural networks respectively.}
	\label{tab:all_results}
	\centering
	{\scriptsize
		\begin{tabular}{@{}l@{\hspace{0.05in}}l@{\hspace{0.05in}}l@{\hspace{0.05in}}l@{\hspace{0.05in}}l@{\hspace{0.05in}}l@{\hspace{0.05in}}|l@{\hspace{0.05in}}l@{\hspace{0.05in}}l@{\hspace{0.05in}}l@{\hspace{0.05in}}l@{\hspace{0.05in}}|l@{\hspace{0.05in}}}
			\toprule
			&Problem & Loss function & Dataset & Metric & Norm &No error & Random & FGSM & PGD & TRADES & Proposed \\
			\midrule
			\multirow{6}{*}{\begin{turn}{90}  Warm up\end{turn}} &Regression& Squared loss& BlogFeedback & MSE & Euclidean & 11.66 & 11.66  &11.66 &11.66 & 11.49 & 11.18\\
			&Regression& Squared loss & BlogFeedback & MSE &  $\ell_{\infty}$ & 11.66& 11.66&11.66 & 11.66& 11.47& 11.20\\
			&Classification & Logistic loss & ImageNet & Accuracy & Euclidean &  49.80  & 48.13 & 49.8 &  49.8  & 50.9  & 56.75 \\
			&Classification & Logistic loss & ImageNet & Accuracy & $\ell_{\infty}$ & 49.80 &  45.46 & 49.8 & 49.8 &  52 & 55.34\\
			&Classification & Hinge loss & ImageNet & Accuracy & Euclidean & 47.89 &  46.66 & 47.8 & 47.9 &  52 &  52.31 \\
			&Classification & Hinge loss & ImageNet & Accuracy&  $\ell_{\infty}$ & 47.89 &  49.59  &49.8 & 49.8 &  52 & 52\\ [0.1 cm]\hline 
			\multirow{8}{*}{\begin{turn}{90}  Main results\end{turn}}  &Classification & NN: ReLU & ImageNet & Accuracy&  Euclidean & 70.74 &  70.66  & 53.88  & 54.48 &  71.55 & 76.49\\
			&Classification & NN: Sigmoid & ImageNet & Accuracy&  Euclidean & 72.5  & 71.08    & 63.45  & 75.89  & 67.92 & 73.04\\
			&Graphical Model & Log-likelihood & TCGA & Likelihood & Euclidean & -7984.8 &  -7980.6 & -7596.5  & -7596.5 &  -7603.4 & -7406.1 \\
			&Graphical Model & Log-likelihood & TCGA & Likelihood & $\ell_{\infty}$ & -7984.8  & -7984.4 &  -7888.4  & -7916.2   & -7917.2 & -7918.7 \\
			&Matrix Completion & Squared loss & Netflix & MSE & Frobenius & 4.78  &  4.89  & 4.6& 3.9  &  4.2   & 3.2\\
			&Matrix Completion & Squared loss & Netflix & MSE & Entry-wise $\ell_{\infty}$ &4.78  &  4.78 &4.57 &3.87  &  4.26 &  3.86\\
			&Max-Margin MC & Squared loss & HouseRep & Accuracy & Frobenius &97.05& 89.72&   96.27&   97.02 &   63.22 &   97.14\\
			&Max-Margin MC & Squared loss & HouseRep & Accuracy & Entry-wise $\ell_{\infty}$ &92.4  &  60.7 & 73.28 &  75.46 &   39.99 &92.5\\
			\bottomrule
	\end{tabular}}
\end{table*}

\subsection{Real-World Experiments}
\label{sec:exp}

As a sanity check, we compare the proposed approach on real-world datasets against five methods of having no adversary, a random adversary \citep{gilmer2019adversarial, qin2021random}, and other well-known methods such as FGSM \citep{goodfellow2014explaining},   projected gradient descent (PGD), and TRADES \citep{zhang2019theoretically}. 

While we could have used $\epsilon$-perturbed test data (coming from the same distribution of the training data) with some synthetic adversary, we preferred to use a more challenging scenario: test data coming from a different distribution than the training data. 

The results are summarized in Table \ref{tab:all_results}, and it can be clearly observed that the proposed method performs better (most of the time) as compared to other methods. Table \ref{tab:all_exp_time_new} shows run time of the proposed method is comparable to the one of other methods (most of the time). Please see Appendix \ref{sec:appn_exp}.

\section{Concluding Remarks}
\label{sec:conclusion}

We proposed a robust adversarial training framework which can be integrated with widely used ML models without any significant computational overhead. As adversarial attacks are not limited only to the problems covered in this work, our analysis can be extended to other problems such as clustering, discrete optimization problems, and randomized algorithms in the future.

\bibliographystyle{plainnat}
\bibliography{mybibfile1}

\newpage
\appendix

\section{Proofs of Section \ref{sec:warmup}: Warm Up} \label{sec:appn_WU_proofs}

\subsection{Proof of Theorem \ref{thm:reg_sqerr}}
\label{sec:proof_linreg}
 \begin{proof}
	Please refer to Lemma \ref{lem:l2eq0_p} and Lemma \ref{lem:l2neq0_p} for the case of $\mathbf{w}^{\intercal}\mathbf{x}^{(i)} - y^{(i)} = 0 $ and  $\mathbf{w}^{\intercal}\mathbf{x}^{(i)} - y^{(i)} \neq 0 $ respectively. The proof relies on norm duality and sub-differentials. 
\end{proof}
\begin{lemma}
	\label{lem:l2eq0_p}
 For the problem specified in Eq.  \eqref{eq:l2loss_lpconst}, the optimal solution for the case when $\mathbf{w}^{\intercal}\mathbf{x}^{(i)} - y^{(i)} = 0 $ is given by $\mathbf{\Delta}^*  = \pm\epsilon \frac{\mathbf{v}}{\norm{\mathbf{v}}}$ where $\mathbf{v} \in \partial\norm{\mathbf{w}}_*  $ as specified in Definition \ref{def:subdiff}. 
\end{lemma}
\begin{proof}
	The problem specified in Eq.  \eqref{eq:l2loss_lpconst} reduces to the dual norm problem for $\mathbf{w}^{\intercal}\mathbf{x}^{(i)} - y^{(i)} = 0 $: 
	\begin{align*}
\sup_{\lmod \mathbf{\Delta} \rmod \leq \epsilon} \mathbf{w}^{\intercal} \mathbf{\Delta}
	\end{align*}
	Using Holder's inequality, we can claim: 
	\begin{align*}
	\mathbf{w}^{\intercal} \mathbf{\Delta} \leq \norm{\mathbf{w}}_* \norm{\mathbf{\Delta}} \leq  \epsilon \norm{\mathbf{w}}_*
	\end{align*}
 Therefore to compute $\mathbf{\Delta}^*$, we need to find the solution for 
	\begin{align*}
	\mathbf{\Delta}^*= \{ \mathbf{\Delta}: \ip{\mathbf{\Delta}}{\mathbf{w}} = \norm{\mathbf{w}}, \norm{\mathbf{\Delta}}_* \leq 1 \}
	\end{align*}
	To compute the optimal point, we use the sub-differential of a norm in Definition \ref{def:subdiff} and claim that  $\mathbf{\Delta}^*  \in \partial\norm{\mathbf{w}}_* $. The scaling is done to maintain $\norm{\mathbf{\Delta}^* } \leq \epsilon$. As the original objective function is quadratic, $-\mathbf{\Delta}^* $ can also be a solution.
\end{proof}

\begin{lemma}
	\label{lem:l2neq0_p}
	For the problem specified in Eq.  \eqref{eq:l2loss_lpconst}, the optimal solution for the case when $\mathbf{w}^{\intercal}\mathbf{x}^{(i)} - y^{(i)} \neq 0 $ is:
	\begin{align*}
	\mathbf{\Delta}^{\star} = \epsilon \sign(\mathbf{w}^{\intercal}\mathbf{x}^{(i)} - y^{(i)} ) \frac{\mathbf{v}}{\norm{\mathbf{v}}} 
	\end{align*}
where	$\mathbf{v} \in \partial\norm{\mathbf{w}}_*  $ as specified in Definition \ref{def:subdiff}. 
\end{lemma}
\begin{proof}
	Assume $\mathbf{w}^{\intercal}\mathbf{x}^{(i)} - y^{(i)} > 0$, so the objective function to maximize $\left( \mathbf{w}^{\intercal} \mathbf{x}^{(i)} - y^{(i)} + \mathbf{w}^{\intercal} \mathbf{\Delta}  \right)^2 $ can be expressed as maximizing $\mathbf{w}^{\intercal} \mathbf{\Delta} $ because $\mathbf{w}^{\intercal} \mathbf{x}^{(i)} - y^{(i)} $ is a positive constant and not a function of $\mathbf{\Delta} $. Further, we use Lemma \ref{lem:l2eq0_p} to derive the solution as $\epsilon \frac{\mathbf{v}}{\norm{\mathbf{v}}} $. 
	
	Similarly for the other case, assuming $\mathbf{w}^{\intercal}\mathbf{x}^{(i)} - y^{(i)} < 0$, our objective is to minimize $\mathbf{w}^{\intercal} \mathbf{\Delta} $  and hence using Lemma \ref{lem:l2eq0_p} or norm duality, the optimal solution is  $-\epsilon \frac{\mathbf{v}}{\norm{\mathbf{v}}} $. Combining the results from the two cases, we complete the proof of this lemma. 
	\end{proof}

\subsection{Proof of Theorem \ref{thm:logit}}
\label{sec:proof_logitthm}
\begin{proof}
The objective function can be seen as maximizing $\log(1 + \exp(-f(\mathbf{\Delta})))$, where $f(\mathbf{\Delta}) = y^{(i)} \mathbf{w}^{\intercal} \left(\mathbf{x}^{(i)} + \mathbf{\Delta}\right) $. It should be noted that $\log(\cdot)$ and $1 + \exp(.)$ are strictly monotonically increasing functions and hence maximizing $\log(1 + \exp(-f(\mathbf{\Delta})))$ is equivalent to minimizing $f(\mathbf{\Delta})$. This is equivalent to minimizing $y^{(i)} \mathbf{w}^{\intercal} \mathbf{\Delta}$. Using Lemma \ref{lem:l2eq0_p}, the solution can be stated as $\mathbf{\Delta}^*  = - \epsilon y^{(i)} \frac{\mathbf{v}}{\norm{\mathbf{v}}}$ where $\mathbf{v} \in \partial\norm{\mathbf{w}}_*  $.
\end{proof}

\subsection{Hinge Loss}
\label{sec:hinge} 
Machine learning models such as support vector machines (SVM) utilize it for various applications involving classification such as text categorization \citep{joachims1998text} and fMRI image classification \citep{gaonkar2013analytic}. Again, using previously introduced notations, we formulate our problem as: 
\begin{align}
\mathbf{\Delta}^{\star} = \argsup_{\lmod \mathbf{\Delta} \rmod \leq \epsilon} \;\; \max\left(0,1 - y^{(i)} \mathbf{w}^{\intercal} \left( \mathbf{x}^{(i)} + \mathbf{\Delta} \right) \right) 
\label{eq:hinge_lpconst}
\end{align}
where $y^{(i)} \in \left\{-1,1 \right\}$ and $\mathbf{x}^{(i)}, \mathbf{\Delta} \in \mathbb{R}^d$. The optimal solution to this problem is proposed in the following theorem. 
\begin{theorem}
	\label{thm:hinge}
	For any general norm $\|\cdot\|$, and the problem specified in Eq. \eqref{eq:hinge_lpconst}, the optimal solution is given by $\mathbf{\Delta}^{\star} = - \epsilon y^{(i)} \nicefrac{\mathbf{v}}{\norm{\mathbf{v}}} $,
 where $\mathbf{v} \in \partial\norm{\mathbf{w}}_*  $ as specified in Definition \ref{def:subdiff}. 
\end{theorem}
\begin{proof}
	Define the function $f(\mathbf{\Delta}) = y^{(i)} \mathbf{w}^{\intercal} \left( \mathbf{x}^{(i)} + \mathbf{\Delta} \right) $. Hence the optimization problem can be seen as the maximization of $\max(0,1-f(\mathbf{\Delta})$. If we had the maximization problem of $1-f(\mathbf{\Delta})$ instead of $\max(0,1-f(\mathbf{\Delta}))$, the solution would be simple. This can be seen as maximization of $1 - y^{(i)} \mathbf{w}^{\intercal} \left( \mathbf{x}^{(i)} + \mathbf{\Delta} \right)$, which is equivalent to minimizing $y^{(i)} \mathbf{w}^{\intercal}  \mathbf{\Delta} $. Using Lemma \ref{lem:l2eq0_p}, the solution can be claimed as $-\epsilon \frac{\mathbf{v}}{\norm{\mathbf{v}}} $. Due to the presence of the $\max$ function in the hinge loss, $(\mathbf{x}^{(i)}, y^{(i)})$ satisfying $\quad y^{(i)}\mathbf{w}^{\intercal} \mathbf{x}^{(i)} -  \epsilon \mathbf{w}^{\intercal}\mathbf{v}\geq 1 $ does not affect the objective function.

\end{proof}

\section{Proofs of Section \ref{sec:main}: Main Results} \label{sec:appn_MR_proofs}

\subsection{Proof of Theorem \ref{thm:Gaussmdl_l2}}
\label{sec:proof_Gaussmdl}
\begin{proof}
First we drop the superscript from $\mathbf{x}^{(i)}$ for clarity. For the problem in Eq. \eqref{eq:Graph_lpconst1} with constraint $\lmod \mathbf{\Delta} \rmod_2 \leq \epsilon$, we write the Lagrangian function: 
\begin{align*}
L(\mathbf{\Delta}, \mu) &= -\frac{1}{2} \left(  \mathbf{\Delta}^{\intercal} \mathbf{\Omega}  \mathbf{\Delta} + 2 \mathbf{\Delta}^{\intercal} \mathbf{\Omega} \mathbf{x}\right) + \frac{\mu}{2}\left( \mathbf{\Delta}^{\intercal} \mathbf{\Delta}  - \epsilon^2 \right) \\
&=\frac{1}{2} \mathbf{\Delta}^{\intercal} \left(\mu \mathbf{I} - \mathbf{\Omega}\right) \mathbf{\Delta} - \mathbf{\Delta}^{\intercal} \mathbf{\Omega} \mathbf{x} - \frac{\mu \epsilon^2}{2}
\end{align*}
Note that we need $\mu \mathbf{I} - \mathbf{\Omega} \succeq 0$ for the problem to be convex. If $\mu \mathbf{I} - \mathbf{\Omega} \succeq 0$ does not hold, then $\left(\mu \mathbf{I} - \mathbf{\Omega}\right) $ has at least one negative eigenvalue. Let $\nu$ and $\mathbf{u}$ be the associated eigenvalue and eigenvector. Therefore, the Lagrangian can be simplified to $L(\mathbf{\Delta}, \mu) = \nu \frac{t^2}{2} - t \mathbf{u}^{\intercal} \mathbf{\Omega x} + c$, where $c$ is a constant by choosing $\mathbf{\Delta} = t \mathbf{u}$. Further, we can set $t \to \infty$ if $\mathbf{u}^{\intercal} \mathbf{\Omega x} > 0$, or $t \to -\infty$ otherwise. Thus $g(\mu) = \inf_{\mathbf{\Delta}}L(\mathbf{\Delta},\mu) = -\infty$. 

Assume $\mu \mathbf{I} - \mathbf{\Omega} \succeq 0$, then by the first order stationarity condition: 
\begin{align}
\frac{\partial L}{\partial \mathbf{\Delta}} = \left(\mu \mathbf{I} - \mathbf{\Omega}\right) \mathbf{\Delta}  - \mathbf{\Omega} \mathbf{x} = \mathbf{0} 
\end{align}
which gives the optimal solution: $  \mathbf{\Delta}^{\star}= \left(\mu \mathbf{I} - \mathbf{\Omega}\right)^{-1}\mathbf{\Omega \mathbf{x}}$.

The dual function, assuming $\mu \mathbf{I} - \mathbf{\Omega} \succeq 0$ is: 
\begin{align}
g(\mu) &= L(\mathbf{\Delta}^{\star} , \mu) \nonumber \\
&= \frac{1}{2}\mathbf{x}^{\intercal} \mathbf{\Omega} \left(\mu \mathbf{I} - \mathbf{\Omega}\right)^{-1} \left(\mu \mathbf{I} - \mathbf{\Omega}\right) \left(\mu \mathbf{I} - \mathbf{\Omega}\right)^{-1}\mathbf{\Omega \mathbf{x}} - \mathbf{x}^{\intercal} \mathbf{\Omega} \left(\mu \mathbf{I} - \mathbf{\Omega}\right)^{-1}\mathbf{\Omega} \mathbf{x}  - \frac{\mu \epsilon^2}{2} \nonumber \\
&= \frac{1}{2}\mathbf{x}^{\intercal} \mathbf{\Omega} \left(\mu \mathbf{I} - \mathbf{\Omega}\right)^{-1}\mathbf{\Omega \mathbf{x}} - \mathbf{x}^{\intercal} \mathbf{\Omega} \left(\mu \mathbf{I} - \mathbf{\Omega}\right)^{-1}\mathbf{\Omega} \mathbf{x}  - \frac{\mu \epsilon^2}{2} \nonumber \\
&= -\frac{1}{2}\mathbf{x}^{\intercal} \mathbf{\Omega} \left(\mu \mathbf{I} - \mathbf{\Omega}\right)^{-1}\mathbf{\Omega \mathbf{x}}- \frac{\mu \epsilon^2}{2} \nonumber
\end{align}
Thus, we have 
\begin{align*}
g(\mu) = \begin{cases}
-\frac{1}{2}\mathbf{x}^{\intercal} \mathbf{\Omega} \left(\mu \mathbf{I} - \mathbf{\Omega}\right)^{-1}\mathbf{\Omega \mathbf{x}}- \frac{\mu \epsilon^2}{2}  & \qquad \mu \mathbf{I} - \mathbf{\Omega} \succeq 0 \\
-\infty & \qquad \text{otherwise}
\end{cases}
\end{align*}

Hence the dual problem is: 
\begin{align*}
\max& \quad -\frac{1}{2}\mathbf{x}^{\intercal} \mathbf{\Omega} \left(\mu \mathbf{I} - \mathbf{\Omega}\right)^{-1}\mathbf{\Omega \mathbf{x}}- \frac{\mu \epsilon^2}{2}  \\
\text{such that} & \quad  \mu \mathbf{I} - \mathbf{\Omega} \succeq 0
\end{align*}
This is a one dimensional optimization problem, which can be solved easily. 
\end{proof}

\subsection{Proof of Theorem \ref{thm:Gaussmdl_linf}}
\label{sec:proof_gaussmdllinf}
\begin{proof}
First we drop the superscript from $\mathbf{x}^{(i)}$ for clarity. The constraint $\norm{\mathbf{\Delta}}_{\infty} \leq \epsilon$ can be expressed as $\max_{i \in [p]} |\mathbf{\Delta}_i| \leq \epsilon$, which implies $\max_{i,j \in [p]} |\mathbf{\Delta}_i \mathbf{\Delta}_j | \leq \epsilon^2$. 

The objective function can be expressed as follows by using the notation for inner products of matrices:
\begin{align}
L(\mathbf{\Delta}) &= \left(  \mathbf{\Delta}^{\intercal} \mathbf{\Omega}  \mathbf{\Delta} + 2 \mathbf{\Delta}^{\intercal} \mathbf{\Omega} \mathbf{x}\right) \nonumber \\
& = \left\langle \begin{bmatrix}
\mathbf{\Omega} & \mathbf{\Omega x} \\
\left( \mathbf{\Omega x}\right)^{\intercal} & 0
\end{bmatrix}, \begin{bmatrix}
\mathbf{\Delta}\mathbf{\Delta}^{\intercal} & \mathbf{\Delta} \\
\mathbf{\Delta}^{\intercal} & 1 
\end{bmatrix}\right\rangle \nonumber \\
& = \left\langle \begin{bmatrix}
\mathbf{\Omega} & \mathbf{\Omega x} \\
\left( \mathbf{\Omega x}\right)^{\intercal} & 0
\end{bmatrix}, \begin{bmatrix}
\mathbf{\Delta} \\  1 
\end{bmatrix} 
\begin{bmatrix}
\mathbf{\Delta} \\  1 
\end{bmatrix}^{\intercal}\right\rangle \label{eq:gauss_loss_sim}
\end{align} 
Hence the above problem can be formulated as an SDP: 
\begin{align*}
\max \quad & \left\langle \begin{bmatrix}
\mathbf{\Omega} & \mathbf{\Omega x} \\
\left( \mathbf{\Omega x}\right)^{\intercal} & 0
\end{bmatrix}, \mathbf{Y} \right\rangle \\ 
\text{such that} \quad & \mathbf{Y}_{p+1, p+1} = 1\\
& \mathbf{Y}_{ij} \leq \epsilon^2, \qquad \forall i,j \in [p]  \\
& -\mathbf{Y}_{ij} \leq \epsilon^2, \qquad \forall i,j \in [p]  \\
& \mathbf{Y} \succeq 0
\end{align*}
The optimal $\mathbf{\Delta}$ can be obtained from the last column/row of the optimal solution $\mathbf{Y}$ for the above problem. 
\end{proof}
\subsection{Proof of Theorem \ref{thm:mat_compl_fro}}
\label{sec:proof_matcomplfro}
\begin{proof}
The function is maximized if the adversary spend the budget on the entries in $\mathcal{P}$. First, we write the Lagrangian for  Eq. \eqref{eq:mat_compl_1}: 
\begin{align*}
L(\mathbf{\Delta}, \lambda) &= -\frac{1}{2}\sum_{(i,j) \in \mathcal{P}}\left(\mathbf{X}_{ij} + \mathbf{\Delta}_{ij} - \mathbf{Y}_{ij} \right)^2 + \frac{\lambda}{2} \left( \sum_{(i,j) \in \mathcal{P}} \mathbf{\Delta}^2_{ij}  - \epsilon^2\right) \\
&= -\frac{1}{2} \sum_{(i,j) \in \mathcal{P}} \left(\mathbf{X}_{ij} - \mathbf{Y}_{ij} \right)^2 -  \sum_{(i,j) \in \mathcal{P}} \mathbf{\Delta}_{ij} \left(\mathbf{X}_{ij}- \mathbf{Y}_{ij} \right) + \frac{-1 + \lambda}{2} \left( \sum_{(i,j) \in \mathcal{P}} \mathbf{\Delta}_{ij}^2 \right)-\frac{\epsilon^2}{2}
\end{align*}
For $\lambda < 1$, we can set $\mathbf{\Delta}_{ij} = t\sign(\mathbf{X}_{ij} - \mathbf{Y}_{ij})$, then the Lagrangian simplifies to: 
\begin{align*}
L(\mathbf{\Delta}, \lambda) = -\frac{1}{2} \sum_{(i,j) \in \mathcal{P}} \left(\mathbf{X}_{ij} - \mathbf{Y}_{ij} \right)^2 -  \sum_{(i,j) \in \mathcal{P}} t \left|\mathbf{X}_{ij}- \mathbf{Y}_{ij} \right| + \frac{-1 + \lambda}{2} \left( \sum_{(i,j) \in \mathcal{P}} t^2 \right)-\frac{\epsilon^2}{2}
\end{align*}
Then we can take $t \to \infty$ and therefore $g(\lambda) = \inf_{\mathbf{\Delta}} L(\mathbf{\Delta}, \lambda) = -\infty$. 

\noindent
For $\lambda >  1$, the first order derivative of the Lagrangian is: 
\begin{align*}
    \frac{\partial L}{\partial \mathbf{\Delta}_{ij}} &= -\left(\mathbf{X}_{ij} + \mathbf{\Delta}_{ij} - \mathbf{Y}_{ij} \right) + \lambda  \mathbf{\Delta}_{ij}  = 0
\end{align*}
and thus:
\begin{align*}
\mathbf{\Delta}^{\star}_{ij} = \frac{\left(\mathbf{X}_{ij} - \mathbf{Y}_{ij} \right) }{\lambda -1} \qquad \text{if } \lambda>1 
\end{align*}
 Hence, the dual function can be derived assuming $\lambda > 1$: 
\begin{align*}
g(\lambda) = L(\mathbf{\Delta}^{\star}, \lambda) = -\frac{1}{2} \frac{\lambda}{\lambda - 1}\sum_{(i,j) \in \mathcal{P}}\left(\mathbf{X}_{ij} - \mathbf{Y}_{ij} \right)^2 -\frac{\lambda \epsilon^2}{2}
\end{align*}
For $\lambda = 1$, the Lagrangian $L$ is: 
\begin{align*}
L(\mathbf{\Delta}, 1) & = -\frac{1}{2} \sum_{(i,j) \in \mathcal{P}} \left(\mathbf{X}_{ij} - \mathbf{Y}_{ij} \right)^2 -  \sum_{(i,j) \in \mathcal{P}} \mathbf{\Delta}_{ij} \left(\mathbf{X}_{ij}- \mathbf{Y}_{ij} \right) -\frac{\epsilon^2}{2}
\end{align*}
Note that $g(1) = \inf_{\mathbf{\Delta}} L(\mathbf{\Delta}, 1) = -\infty$ if there exists $(i,j) \in \mathcal{P}$ such that $ \mathbf{X}_{ij} \neq \mathbf{Y}_{ij}$, since we can set $\mathbf{\Delta}_{ij}= t \sign(\mathbf{X}_{ij}- \mathbf{Y}_{ij})$ and take $t \to \infty$.

Thus we have the dual function as: 
\begin{align*}
g(\lambda) = \begin{cases}
-\frac{1}{2} \frac{\lambda}{\lambda - 1}\sum_{(i,j) \in \mathcal{P}}\left(\mathbf{X}_{ij} - \mathbf{Y}_{ij} \right)^2 -\frac{\lambda \epsilon^2}{2} & \qquad \lambda > 1 \\
-\frac{ \epsilon^2}{2} & \qquad \lambda = 1 \text{ and } \mathbf{X}_{ij} = \mathbf{Y}_{ij}, \forall (i,j) \in \mathcal{P}\\
-\infty & \qquad \text{otherwise}
\end{cases}
\end{align*}
The optimal $\lambda$ can be derived by taking the first order derivative 
\begin{align*}
\frac{1}{\epsilon^2}\sum_{(i,j) \in \mathcal{P}} \left(\mathbf{X}_{ij} - \mathbf{Y}_{ij} \right)^2 = (\lambda - 1)^2 
\end{align*}
Therefore 
\begin{align*}
\lambda^{\star} = 1 + \frac{1}{\epsilon} \sqrt{\sum_{(i,j) \in \mathcal{P}} \left(\mathbf{X}_{ij} - \mathbf{Y}_{ij} \right)^2}
\end{align*}
Hence $\mathbf{\Delta}^{\star}_{ij}$ is given by
\begin{align*}
\mathbf{\Delta}^{\star}_{ij} = \epsilon \frac{\left(\mathbf{X}_{ij} - \mathbf{Y}_{ij} \right) }{\sqrt{\sum_{(i,j) \in \mathcal{P}} \left(\mathbf{X}_{ij} - \mathbf{Y}_{ij} \right)^2}}
\end{align*}
If $\mathbf{X}_{ij} = \mathbf{Y}_{ij}, \forall (i,j) \in \mathcal{P}$, then the optimal $\mathbf{\Delta}^{\star}$ can be any solution satisfying $\sum_{(i,j) \in \mathcal{P}} \mathbf{\Delta}^2_{ij} = \epsilon$. 
\end{proof}

\subsection{Proof of Corollary \ref{cor:matrixcompl_linf}}
\begin{proof}
For this case with $|\mathbf{\Delta}_{ij}| \leq \epsilon$ for all  $(i,j) \in \mathcal{P}$, the problem can be solved for each $\mathbf{\Delta}_{ij} $ separately. The problem for each $\mathbf{\Delta}_{ij} $ separately reduces to a particular case of Lemma \ref{lem:l2neq0_p} in Appendix \ref{sec:proof_linreg}.
\end{proof}

\subsection{Proof of Theorem \ref{thm:maxmarginmat_fro}}
\label{sec:proof_maxmargin_fro}
\begin{proof}
The problem without the max function in the objective function is equivalent to: 
\begin{align*}
\sup_{\lmod \mathbf{\Delta} \rmod^2_F \leq \epsilon^2} -\sum_{(i,j) \in \mathcal{P}} \mathbf{\Delta}_{ij} \mathbf{Y}_{ij}
\end{align*}
The optimal solution for the above problem is 
\begin{align}
\mathbf{\Delta}_{ij} = -\mathbf{Y}_{ij} \frac{\epsilon}{\sqrt{\sum\limits_{(i,j) \in \mathcal{P} } \mathbf{Y}^2_{ij} }} \nonumber
\end{align}
for $(i,j) \in \mathcal{P}$. But the optimal solution changes with the introduction of the max term in Eq. \eqref{eq:maxmarginmat_prbdef}. Few of the terms with indices $(i,j) \in \mathcal{P}$ do not affect the objective function if 
\begin{align}
\mathbf{X}_{ij} \mathbf{Y}_{ij} > 1 + \frac{\epsilon}{\sqrt{\sum\limits_{(i,j) \in \mathcal{P} } \mathbf{Y}^2_{ij} }} \nonumber
\end{align}
Hence the budget should be spent on the indices satisfying $\mathbf{X}_{ij} \mathbf{Y}_{ij} < 1 + \frac{\epsilon}{\sqrt{\sum\limits_{(i,j) \in \mathcal{P}_1 } \mathbf{Y}^2_{ij} }}$, where $\mathcal{P}_1 \subseteq \mathcal{P}$ is the modified set. Note that the set $\mathcal{P}_1$ can be derived by sorting $\mathbf{X}_{ij} \mathbf{Y}_{ij} $ for $(i,j) \in \mathcal{P}$, which takes $\mathcal{O}(|\mathcal{P}| \log(|\mathcal{P}|)$ time. 

We further describe the method to compute $\mathcal{P}_1$. Let $\mathbf{Z} = \mathbf{X} \odot \mathbf{Y}$ denote the Hadamard product of $\mathbf{X}$ and $\mathbf{Y}$. We define the mapping $\Pi: \{1,2,\ldots,|\mathcal{P}| \} \rightarrow \mathcal{P}$ which sorts the terms $\mathbf{X}_{ij} \mathbf{Y}_{ij}$ for $(i,j) \in \mathcal{P}$ in ascending order, i.e. $\mathbf{Z}_{\Pi(1)} \leq \mathbf{Z}_{\Pi(2)} \leq \ldots \mathbf{Z}_{\Pi(n)}$,
 where $n = |\mathcal{P}|$. 

Now consider the three cases: 
\begin{enumerate}
    \item Case 1: Assume $\mathbf{Z}_{\Pi(1)} \geq 1 + \epsilon$. Therefore, $\mathbf{Z}_{\Pi(i)} \geq 1 + \epsilon$ for all $i \in [n]$. Hence, any change in $\mathbf{X}_{ij} \mathbf{Y}_{ij} $ does not make any change in the objective function. Therefore, $\mathbf{\Delta}_{ij} = 0$ for all $(i,j) \in \mathcal{P}$ and $\mathcal{P}_1 = \emptyset$.

\item Case 2: Assume $\mathbf{Z}_{\Pi(n)} \leq 1$. All the $\mathbf{X}_{ij} \mathbf{Y}_{ij}$ can be decreased to increase the objective function value. Thus, $\mathbf{\Delta}_{ij} = -\epsilon \frac{\mathbf{Y}_{ij}}{{\sqrt{\sum\limits_{(i,j) \in \mathcal{P} } \mathbf{Y}^2_{ij} }}}$ and $\mathcal{P}_1 = \mathcal{P}$. 

\item Case 3: Other cases which are not satisfied in the above two cases are discussed here. We define the left set $\mathcal{S}_l = \{ \Pi(i) | Z_{\Pi(i)} \leq 1, i \in [n] \}$. 
We also define the middle set $\mathcal{S}_m = \{ \Pi(i) | Z_{\Pi(i)} \in (1,1+\epsilon), i \in [n] \} $ and let $k = |\mathcal{S}_l| > 0$. A few elements of the set $\mathcal{S}_m$ will contribute in decreasing the objective function and we discuss the approach to compute those terms. 

Consider two sub-cases: 
\begin{enumerate}
    \item If $\mathbf{Z}_{\Pi(k+1)} > 1 + \nicefrac{\epsilon}{\sqrt{\sum\limits_{(i,j) \in \mathcal{S}_{k+1} } \mathbf{Y}^2_{ij} }}$, where $\mathcal{S}_{k+1} = \mathcal{S}_l \cup \Pi(k+1)$, then ${\Pi(k+1)}$ should not be included and hence $\mathcal{P}_1 = \mathcal{S}_l$ and each $\mathbf{\Delta}_{ij} = -Y_{ij} \frac{\epsilon}{\sqrt{\sum\limits_{(i,j) \in \mathcal{P}_1 } \mathbf{Y}^2_{ij} }}$. 
    \item If $\mathbf{Z}_{\Pi(k+1)} \leq 1 + \nicefrac{\epsilon}{\sqrt{\sum\limits_{(i,j) \in \mathcal{S}_{k+1} } \mathbf{Y}^2_{ij} }}$, then ${\Pi(k+1)}$ should be included in $\mathcal{P}_1$. 
    Assume such $i^{\star}$ elements can be included in $\mathcal{P}_1$. This can be computed by finding the largest index $i^{\star} \in \{ k+1, k+2, \dots, k+|S_m|\}$ such that $\mathbf{Z}_{\Pi(i)} \leq 1 + \frac{\epsilon}{\sqrt{\sum\limits_{(i,j) \in \mathcal{P}_1 } \mathbf{Y}^2_{ij} }}$, where  $\mathcal{P}_1 = \{ \Pi(i) | i \in \{1, 2, \ldots, i^{\star} \}$ and $\mathbf{\Delta}_{ij} $ can be computed as $-\mathbf{Y}_{ij} \frac{\epsilon}{\sqrt{\sum\limits_{(i,j) \in \mathcal{P}_1 } \mathbf{Y}^2_{ij} }}$. 
\end{enumerate}
\end{enumerate}
\end{proof}

\subsection{Proof for Corollary \ref{cor:maxmarginmat_linf}}
\begin{proof}
	This is a particular case of Lemma \ref{lem:l2neq0_p} in Appendix \ref{sec:proof_linreg}. As all the entries $(i,j) \in \mathcal{P}$ of $\mathbf{\Delta}$ can use the budget $\epsilon$ independently, the problem can be solved for each $\mathbf{\Delta}_{ij}$ separately. 
\end{proof}
\section{Proofs of Section \ref{sec:rademacher}: Novel Rademacher Complexities} 
\label{sec:appn_RC_proof}

\subsection{Proof of Theorem \ref{thm:RC_linreg}}
\begin{proof}
    Using Definition \ref{def:emp_RC}: 
    \begin{align*}
        \hat{\mathfrak{R}}_S(\mathcal{F}) =  \E_{\sigma} \sqbr{\sup_{\norm{\mathbf{w}}_{*} \leq C}\nrbr{ \frac{1}{n} \sum_{i = 1}^n \sigma_i   \left( \mathbf{w}^{\intercal}  \mathbf{x}^{(i)}  - y^{(i)}\right)^2 }}
    \end{align*}
    Similarly, in the case of adversarial perturbation: 
    \begin{align}
        \hat{\mathfrak{R}}_S(\Tilde{\mathcal{F}}) &= \E_{\sigma} \sqbr{\sup_{\norm{\mathbf{w}}_{*} \leq C}\nrbr{ \frac{1}{n} \sum_{i = 1}^n \sigma_i  \sup\limits_{\norm{\mathbf{\Delta}} \leq \epsilon}\left( \mathbf{w}^{\intercal} \left( \mathbf{x}^{(i)} + \mathbf{\Delta} \right) - y^{(i)}\right)^2 }} \nonumber\\
        &= \E_{\sigma} \sqbr{\sup_{\norm{\mathbf{w}}_{*} \leq C}\nrbr{ \frac{1}{n} \sum_{i = 1}^n \sigma_i  \sup\limits_{\norm{\mathbf{\Delta}} \leq \epsilon} \nrbr{  \left( \mathbf{w}^{\intercal} \mathbf{x}^{(i)}  - y^{(i)}\right)^2 + 2  \left( \mathbf{w}^{\intercal} \mathbf{x}^{(i)}  - y^{(i)}\right) \mathbf{w}^{\intercal} \mathbf{\Delta} + \nrbr{\mathbf{w}^{\intercal} \mathbf{\Delta}}^2}} } \label{eq:RC_linreg_eq}
    \end{align}
    For the upper bound, we use $\sup(A+B) \leq \sup(A) + \sup(B)$ to claim: 
    \begin{align*}
        \hat{\mathfrak{R}}_S(\Tilde{\mathcal{F}}) \leq & \E_{\sigma} \sqbr{\sup_{\norm{\mathbf{w}}_{*} \leq C}\nrbr{ \frac{1}{n} \sum_{i = 1}^n \sigma_i  \sup\limits_{\norm{\mathbf{\Delta}} \leq \epsilon} \nrbr{  \left( \mathbf{w}^{\intercal} \mathbf{x}^{(i)}  - y^{(i)}\right)^2  }}} \\ 
        & + \E_{\sigma} \sqbr{\sup_{\norm{\mathbf{w}}_{*} \leq C}\nrbr{ \frac{1}{n} \sum_{i = 1}^n \sigma_i  \sup\limits_{\norm{\mathbf{\Delta}} \leq \epsilon} \nrbr{   2  \left( \mathbf{w}^{\intercal} \mathbf{x}^{(i)}  - y^{(i)}\right) \mathbf{w}^{\intercal} \mathbf{\Delta} + \nrbr{\mathbf{w}^{\intercal} \mathbf{\Delta}}^2}} } \\
        &= \hat{\mathfrak{R}}_S(\mathcal{F}) + \E_{\sigma} \sqbr{\sup_{\norm{\mathbf{w}}_{*} \leq C}\nrbr{ \frac{1}{n} \sum_{i = 1}^n \sigma_i  \sup\limits_{\norm{\mathbf{\Delta}} \leq \epsilon} \nrbr{ 2  \left( \mathbf{w}^{\intercal} \mathbf{x}^{(i)}  - y^{(i)}\right) \mathbf{w}^{\intercal} \mathbf{\Delta} + \nrbr{\mathbf{w}^{\intercal} \mathbf{\Delta}}^2}} } 
    \end{align*}
    Substituting the optimal $\mathbf{\Delta}^{\star}$ derived in Theorem \ref{thm:reg_sqerr} in the above equation, we claim: 
    \begin{align*}
        \hat{\mathfrak{R}}_S(\Tilde{\mathcal{F}}) &= \hat{\mathfrak{R}}_S(\mathcal{F}) + \E_{\sigma} \sqbr{\sup_{\norm{\mathbf{w}}_{*} \leq C}\nrbr{ \frac{1}{n} \sum_{i = 1}^n \sigma_i  \nrbr{ 2 \epsilon \left| \mathbf{w}^{\intercal} \mathbf{x}^{(i)}  - y^{(i)}\right|    \norm{\mathbf{w}}_* \mathds{1}_{(\mathbf{w}^{\intercal} \mathbf{x}^{(i)}  - y^{(i)}) \neq 0} + \epsilon^2 \norm{\mathbf{w}}_*^2}} } \\
        & \leq \hat{\mathfrak{R}}_S(\mathcal{F}) + \E_{\sigma} \sqbr{\sup_{\norm{\mathbf{w}}_{*} \leq C}\nrbr{ \frac{1}{n} \sum_{i = 1}^n  2\sigma_i   \epsilon \left| \mathbf{w}^{\intercal} \mathbf{x}^{(i)}  - y^{(i)}\right|    \norm{\mathbf{w}}_* } } + \E_{\sigma} \sqbr{\sup_{\norm{\mathbf{w}}_{*} \leq C}\nrbr{ \frac{1}{n} \sum_{i = 1}^n  \sigma_i \epsilon^2 \norm{\mathbf{w}}_*^2 } } \\
        & \leq \hat{\mathfrak{R}}_S(\mathcal{F}) + \E_{\sigma} \sqbr{\sup_{\norm{\mathbf{w}}_{*} \leq C}\nrbr{ \frac{1}{n} \sum_{i = 1}^n  2\sigma_i   \epsilon \left| \mathbf{w}^{\intercal} \mathbf{x}^{(i)}  - y^{(i)}\right| }  \sup_{\norm{\mathbf{w}}_{*} \leq C}\norm{\mathbf{w}}_*  } + \E_{\sigma} \sqbr{\frac{1}{n} \left|\sum_{i = 1}^n  \sigma_i \epsilon^2  \right|\sup_{\norm{\mathbf{w}}_{*} \leq C}\norm{\mathbf{w}}_*^2 } 
            \end{align*}
    where $\mathds{1}_{(.)}$ denotes an indicator function. 

Using Ledoux-Talagrand contraction lemma as $|\cdot|$ is 1-Lipschitz:
\begin{align*}
    \hat{\mathfrak{R}}_S(\Tilde{\mathcal{F}}) & \leq \hat{\mathfrak{R}}_S(\mathcal{F}) + \E_{\sigma} \sqbr{\sup_{\norm{\mathbf{w}}_{*} \leq C}\nrbr{ \frac{1}{n} \sum_{i = 1}^n  2\sigma_i   \epsilon  \nrbr{\mathbf{w}^{\intercal} \mathbf{x}^{(i)}  - y^{(i)}}} C  } + \epsilon^2 C^2\E_{\sigma} \sqbr{\frac{1}{n} \left|\sum_{i = 1}^n  \sigma_i   \right|}  \\
    & \leq \hat{\mathfrak{R}}_S(\mathcal{F}) + \E_{\sigma} \sqbr{\sup_{\norm{\mathbf{w}}_{*} \leq C}\nrbr{ \frac{1}{n} \sum_{i = 1}^n  2\sigma_i   \epsilon \mathbf{w}^{\intercal} \mathbf{x}^{(i)}  } C  } + \epsilon^2 C^2\E_{\sigma} \sqbr{\frac{1}{n} \left|\sum_{i = 1}^n  \sigma_i   \right|}  \\
    & \leq \hat{\mathfrak{R}}_S(\mathcal{F}) + 2\epsilon C \E_{\sigma} \sqbr{\norm{ \frac{1}{n} \sum_{i = 1}^n  \sigma_i \mathbf{x}^{(i)}  }   } + \frac{\epsilon^2 C^2}{\sqrt{n}}   
    \end{align*}
We have used Khintchine’s inequality in the last step. This completes the proof for the upper bound. 

To prove the lower bound, we use $\sup(A+B) \geq \sup(A) - \sup(-B)$ in Eq. \eqref{eq:RC_linreg_eq} to claim: 
\begin{align*}
    \hat{\mathfrak{R}}_S(\Tilde{\mathcal{F}}) &\geq \hat{\mathfrak{R}}_S(\mathcal{F}) - \E_{\sigma} \sqbr{\sup_{\norm{\mathbf{w}}_{*} \leq C}\nrbr{ \frac{1}{n} \sum_{i = 1}^n - \sigma_i  \sup\limits_{\norm{\mathbf{\Delta}} \leq \epsilon} \nrbr{ 2  \left( \mathbf{w}^{\intercal} \mathbf{x}^{(i)}  - y^{(i)}\right) \mathbf{w}^{\intercal} \mathbf{\Delta} + \nrbr{\mathbf{w}^{\intercal} \mathbf{\Delta}}^2}} } \\
    &= \hat{\mathfrak{R}}_S(\mathcal{F}) - \E_{\sigma} \sqbr{\sup_{\norm{\mathbf{w}}_{*} \leq C}\nrbr{ \frac{1}{n} \sum_{i = 1}^n  \sigma_i  \sup\limits_{\norm{\mathbf{\Delta}} \leq \epsilon} \nrbr{ 2  \left( \mathbf{w}^{\intercal} \mathbf{x}^{(i)}  - y^{(i)}\right) \mathbf{w}^{\intercal} \mathbf{\Delta} + \nrbr{\mathbf{w}^{\intercal} \mathbf{\Delta}}^2}} } \\
    &\geq  \hat{\mathfrak{R}}_S(\mathcal{F}) - 2\epsilon C \E_{\sigma} \sqbr{\norm{ \frac{1}{n} \sum_{i = 1}^n  \sigma_i \mathbf{x}^{(i)}  }   } - \frac{\epsilon^2 C^2}{\sqrt{n}}  
\end{align*}
where we have substituted $\sigma_i$ as $-\sigma_i$ in the second step as they have same distribution. The subsequent steps are similar to the proof of the upper bound. 

\end{proof}

\subsection{Proof of Theorem \ref{thm:RC_linclass}}
\begin{proof}
Using Definition \ref{def:emp_RC}, we can claim the following for the case of without any adversarial perturbation: 
\begin{align*}
    \hat{\mathfrak{R}}_S(\mathcal{F}) = \E_{\sigma} \sqbr{\sup_{\norm{\mathbf{w}}_{*} \leq C}\nrbr{ \frac{1}{n} \sum_{i = 1}^n -\sigma_i y^{(i)}\mathbf{w}^{\intercal}\mathbf{x}^{(i)} }} = \E_{\sigma} \sqbr{\sup_{\norm{\mathbf{w}}_{*} \leq C}\nrbr{ \frac{1}{n} \sum_{i = 1}^n \sigma_i y^{(i)}\mathbf{w}^{\intercal}\mathbf{x}^{(i)} }}
\end{align*}
where we have replaced $\sigma_i$ with $-\sigma_i$ as they have same distribution. 

Similarly, by using the definition for the case of adversarial perturbation: 
\begin{align*}
    \hat{\mathfrak{R}}_S(\Tilde{\mathcal{F}}) &= \frac{1}{n} \E_{\sigma} \sqbr{\sup_{\norm{\mathbf{w}}_{*} \leq C}\nrbr{  \sum_{i = 1}^n \sigma_i \sup_{\norm{\mathbf{\Delta}} \leq \epsilon} -y^{(i)}\mathbf{w}^{\intercal}(\mathbf{x}^{(i)} +\mathbf{\Delta}) }} \\
    &=\frac{1}{n} \E_{\sigma} \sqbr{\sup_{\norm{\mathbf{w}}_{*} \leq C}\nrbr{  \sum_{i = 1}^n \sigma_i \min_{\norm{\mathbf{\Delta}} \leq \epsilon} y^{(i)}\mathbf{w}^{\intercal}(\mathbf{x}^{(i)} +\mathbf{\Delta}) }}
\end{align*}
Using the optimal $\mathbf{\Delta}^{\star}$ for inner minimization problem derived in Theorem \ref{thm:logit}: 
\begin{align}
    \hat{\mathfrak{R}}_S(\Tilde{\mathcal{F}}) &=\frac{1}{n} \E_{\sigma} \sqbr{\sup_{\norm{\mathbf{w}}_{*} \leq C}\nrbr{  \sum_{i = 1}^n \sigma_i  y^{(i)}\mathbf{w}^{\intercal}\nrbr{\mathbf{x}^{(i)} - \epsilon y^{(i)} \frac{\mathbf{v}}{\norm{\mathbf{v}}}} }} \qquad \qquad \text{where} \; \mathbf{v} \in \partial\norm{\mathbf{w}}_*  \nonumber \\
    &= \frac{1}{n} \E_{\sigma} \sqbr{\sup_{\norm{\mathbf{w}}_{*} \leq C}\nrbr{  \sum_{i = 1}^n \sigma_i  y^{(i)}\mathbf{w}^{\intercal}\mathbf{x}^{(i)}  - \epsilon \sigma_i \mathbf{w}^{\intercal}\frac{\mathbf{v}}{\norm{\mathbf{v}}}} } \nonumber\\
    &= \frac{1}{n} \E_{\sigma} \sqbr{\sup_{\norm{\mathbf{w}}_{*} \leq C}\nrbr{  \sum_{i = 1}^n \sigma_i  y^{(i)}\mathbf{w}^{\intercal}\mathbf{x}^{(i)}  - \epsilon \sigma_i \norm{\mathbf{w}}_*} } \label{eq:linclass_RC1}\\
    & \leq \E_{\sigma} \sqbr{\sup_{\norm{\mathbf{w}}_{*} \leq C}\nrbr{ \frac{1}{n} \sum_{i = 1}^n \sigma_i y^{(i)}\mathbf{w}^{\intercal}\mathbf{x}^{(i)} }} + \frac{1}{n} \E_{\sigma} \sqbr{ \sup_{\norm{\mathbf{w}}_{*} \leq C} \sum_{i = 1}^n -\epsilon \sigma_i \norm{\mathbf{w}}_* } \nonumber \\
    & \leq \hat{\mathfrak{R}}_S(\mathcal{F}) + \frac{1}{n} \E_{\sigma} \sqbr{  \left|\sum_{i = 1}^n -\epsilon\sigma_i \right| \sup_{\norm{\mathbf{w}}_{*} \leq C}  \norm{\mathbf{w}}_* } \nonumber \\
    & \leq \hat{\mathfrak{R}}_S(\mathcal{F}) + \frac{C\epsilon}{n} \E_{\sigma} \sqbr{  \left|\sum_{i = 1}^n \sigma_i \right|} \leq \hat{\mathfrak{R}}_S(\mathcal{F}) + \frac{C\epsilon}{\sqrt{n}} \nonumber
\end{align}
where we have used Khintchine’s inequality in the last step. This completes the proof for the upper bound.

For the lower bound, we can replace $\sigma_i$ with $-\sigma_i$ in Eq. \eqref{eq:linclass_RC1} as they have the same distribution to obtain: 
\begin{align}
    \hat{\mathfrak{R}}_S(\Tilde{\mathcal{F}}) &= \frac{1}{n} \E_{\sigma} \sqbr{\sup_{\norm{\mathbf{w}}_{*} \leq C}\nrbr{  \sum_{i = 1}^n -\sigma_i  y^{(i)}\mathbf{w}^{\intercal}\mathbf{x}^{(i)}  + \epsilon \sigma_i \norm{\mathbf{w}}_*} } \label{eq:linclass_RCminussigma}
\end{align}
Further, we can replace $\mathbf{w}$ with $-\mathbf{w}$ in the above equation to obtain: 
\begin{align}
    \hat{\mathfrak{R}}_S(\Tilde{\mathcal{F}}) &= \frac{1}{n} \E_{\sigma} \sqbr{\sup_{\norm{\mathbf{w}}_{*} \leq C}\nrbr{  \sum_{i = 1}^n \sigma_i  y^{(i)}\mathbf{w}^{\intercal}\mathbf{x}^{(i)}  + \epsilon \sigma_i \norm{\mathbf{w}}_*} } \label{eq:linclass_RCminusw}
\end{align}
Adding Eq. \eqref{eq:linclass_RCminussigma} and Eq. \eqref{eq:linclass_RCminusw}, we arrive at: 
\begin{align}
    \hat{\mathfrak{R}}_S(\Tilde{\mathcal{F}}) & = \frac{1}{2n} \nrbr{\E_{\sigma} \sqbr{\sup_{\norm{\mathbf{w}}_{*} \leq C}\nrbr{  \sum_{i = 1}^n -\sigma_i  y^{(i)}\mathbf{w}^{\intercal}\mathbf{x}^{(i)}  + \epsilon \sigma_i \norm{\mathbf{w}}_*} }  + \E_{\sigma} \sqbr{\sup_{\norm{\mathbf{w}}_{*} \leq C}\nrbr{  \sum_{i = 1}^n \sigma_i  y^{(i)}\mathbf{w}^{\intercal}\mathbf{x}^{(i)}  + \epsilon \sigma_i \norm{\mathbf{w}}_*} }} \nonumber \\
    &\geq \frac{1}{n} \E_{\sigma} \sqbr{ \sup_{\norm{\mathbf{w}}_{*} \leq C} \sum_{i = 1}^n \epsilon \sigma_i \norm{\mathbf{w}}_*}=  \frac{1}{2n} \E_{\sigma} \sqbr{ \sup_{\norm{\mathbf{w}}_{*} \leq C} \left|\sum_{i = 1}^n \epsilon \sigma_i \right|\norm{\mathbf{w}}_*}  = \frac{C\epsilon}{2n} \E_{\sigma} \sqbr{  \left|\sum_{i = 1}^n \sigma_i \right|}  \geq \frac{C\epsilon}{2\sqrt{2n}} \label{eq:linclass_lb_1}
\end{align}
where we have used Khintchine’s inequality in the last step. 

Similarly, adding Eq. \eqref{eq:linclass_RC1} and Eq. \eqref{eq:linclass_RCminusw}, we arrive at: 
\begin{align}
    \hat{\mathfrak{R}}_S(\Tilde{\mathcal{F}}) & = \frac{1}{2n} \nrbr{\E_{\sigma} \sqbr{\sup_{\norm{\mathbf{w}}_{*} \leq C}\nrbr{  \sum_{i = 1}^n \sigma_i  y^{(i)}\mathbf{w}^{\intercal}\mathbf{x}^{(i)}  - \epsilon \sigma_i \norm{\mathbf{w}}_*} } + \E_{\sigma} \sqbr{\sup_{\norm{\mathbf{w}}_{*} \leq C}\nrbr{  \sum_{i = 1}^n \sigma_i  y^{(i)}\mathbf{w}^{\intercal}\mathbf{x}^{(i)}  + \epsilon \sigma_i \norm{\mathbf{w}}_*} }} \nonumber \\
    &\geq \E_{\sigma} \sqbr{\sup_{\norm{\mathbf{w}}_{*} \leq C}\nrbr{ \frac{1}{n} \sum_{i = 1}^n \sigma_i y^{(i)}\mathbf{w}^{\intercal}\mathbf{x}^{(i)} }} =  \hat{\mathfrak{R}}_S(\mathcal{F}) \label{eq:linclass_lb_2}
\end{align}
Combining the results from Eq. \eqref{eq:linclass_lb_1}  and Eq. \eqref{eq:linclass_lb_2}, we arrive at the desired lower bound result.   
\end{proof}

\subsection{Proof of Theorem \ref{thm:RC_MC}}
\begin{proof}
    Using Definition \ref{def:RC_2d}, the empirical Rademacher complexity for the case without any adversarial perturbation can be defined as: 
    \begin{align*}
        \hat{\mathfrak{R}}_{\mathbf{X}}(\mathcal{F}) = \E_{\sigma} \sqbr{\sup_{\norm{\mathbf{Y}}_* \leq C} \frac{1}{|\mathcal{P}|} \sum_{(i,j) \in \mathcal{P}} \sigma_{ij} \nrbr{\mathbf{X}_{ij} - \mathbf{Y}_{ij} }^2} 
    \end{align*}
Similarly for the Rademacher complexity for the case with adversarial perturbation can be defined as: 
\begin{align}
     \hat{\mathfrak{R}}_{\mathbf{X}}(\Tilde{\mathcal{F}}) &= \E_{\sigma} \sqbr{\sup_{\norm{\mathbf{Y}}_* \leq C} \frac{1}{|\mathcal{P}|} \sum_{(i,j) \in \mathcal{P}} \sigma_{ij} \sup\limits_{\norm{\mathbf{\Delta}_{ij}} \leq \epsilon}  \nrbr{\mathbf{X}_{ij} + \mathbf{\Delta}_{ij} - \mathbf{Y}_{ij} }^2} \nonumber \\ 
    &= \E_{\sigma} \sqbr{\sup_{\norm{\mathbf{Y}}_* \leq C} \frac{1}{|\mathcal{P}|} \sum_{(i,j) \in \mathcal{P}} \sigma_{ij}  \nrbr{\mathbf{X}_{ij} + \epsilon\text{sign}\nrbr{\mathbf{X}_{ij} - \mathbf{Y}_{ij}} - \mathbf{Y}_{ij} }^2} \qquad \qquad \text{(using Corollary \ref{cor:matrixcompl_linf})} \nonumber \\ 
    & =\E_{\sigma} \sqbr{\sup_{\norm{\mathbf{Y}}_* \leq C} \frac{1}{|\mathcal{P}|} \sum_{(i,j) \in \mathcal{P}} \nrbr{\sigma_{ij} \nrbr{\mathbf{X}_{ij} - \mathbf{Y}_{ij} }^2 + 2\epsilon \sigma_{ij}\left|\mathbf{X}_{ij} - \mathbf{Y}_{ij} \right| + \sigma_{ij}\epsilon^2} } \nonumber \\ 
    & =\E_{\sigma} \sqbr{\sup_{\norm{\mathbf{Y}}_* \leq C} \frac{1}{|\mathcal{P}|} \sum_{(i,j) \in \mathcal{P}} \nrbr{\sigma_{ij} \nrbr{\mathbf{X}_{ij} - \mathbf{Y}_{ij} }^2 + 2\epsilon \sigma_{ij}\left|\mathbf{X}_{ij} - \mathbf{Y}_{ij} \right| } } + 0 \label{eq:RC_MC_eq1}  \\ 
    & \leq \E_{\sigma} \sqbr{\sup_{\norm{\mathbf{Y}}_* \leq C} \frac{1}{|\mathcal{P}|} \sum_{(i,j) \in \mathcal{P}} \nrbr{\sigma_{ij} \nrbr{\mathbf{X}_{ij} - \mathbf{Y}_{ij} }^2}} + \E_{\sigma} \sqbr{\sup_{\norm{\mathbf{Y}}_* \leq C} \frac{2\epsilon}{|\mathcal{P}|} \sum_{(i,j) \in \mathcal{P}} \sigma_{ij} \left|\mathbf{Y}_{ij} - \mathbf{X}_{ij} \right| }    
    \nonumber 
\end{align}
Using Ledoux-Talagrand contraction lemma as $|\cdot|$ is 1-Lipschitz:
\begin{align}
    \hat{\mathfrak{R}}_{\mathbf{X}}(\Tilde{\mathcal{F}}) &\leq \hat{\mathfrak{R}}_{\mathbf{X}}(\mathcal{F}) + \E_{\sigma} \sqbr{\sup_{\norm{\mathbf{Y}}_* \leq C} \frac{2 \epsilon}{|\mathcal{P}|}  \sum_{(i,j) \in \mathcal{P}} \sigma_{ij} \nrbr{\mathbf{Y}_{ij} - \mathbf{X}_{ij}} }    
    \nonumber \\
    & \leq \hat{\mathfrak{R}}_{\mathbf{X}}(\mathcal{F}) + \E_{\sigma} \sqbr{\sup_{\norm{\mathbf{Y}}_* \leq C} \frac{2 \epsilon}{|\mathcal{P}|} \sum_{(i,j) \in \mathcal{P}} \sigma_{ij} \mathbf{Y}_{ij}  }   + 0  \nonumber \\
    & \leq       \hat{\mathfrak{R}}_{\mathbf{X}}(\mathcal{F}) + \E_{\sigma}  \sqbr{\frac{2\epsilon}{|\mathcal{P}|} \sup_{\norm{\mathbf{Y}}_* \leq C} \norm{\bm{\sigma}}\norm{\mathbf{Y}}_*} \nonumber \\
        & = \hat{\mathfrak{R}}_{\mathbf{X}}(\mathcal{F}) + \frac{2\epsilon C}{|\mathcal{P}|} \E_{\sigma}  \sqbr{ \norm{\bm{\sigma}}} \nonumber
\end{align}
where $\bm{\sigma}$ is a matrix containing $\sigma_{ij}$ for entries $(i,j) \in \mathcal{P}$ and $\norm{\cdot}_{\infty}$ is the entry wise $\ell_{\infty}$ norm.  This completes the proof for the upper bound. 

For the lower bound, we use $\sup(A+B) \geq \sup(A) -\sup(-B)$ and linearity of expectation in Eq. \eqref{eq:RC_MC_eq1} to claim: 
\begin{align}
    \hat{\mathfrak{R}}_{\mathbf{X}}(\Tilde{\mathcal{F}}) & \geq E_{\sigma} \sqbr{\sup_{\norm{\mathbf{Y}}_* \leq C} \frac{1}{|\mathcal{P}|} \sum_{(i,j) \in \mathcal{P}} \nrbr{\sigma_{ij} \nrbr{\mathbf{X}_{ij} - \mathbf{Y}_{ij} }^2 } } - E_{\sigma} \sqbr{\sup_{\norm{\mathbf{Y}}_* \leq C} \frac{1}{|\mathcal{P}|} \sum_{(i,j) \in \mathcal{P}}  -2\epsilon \sigma_{ij}\left|\mathbf{X}_{ij} - \mathbf{Y}_{ij} \right| } \label{eq:RC_lb_MC1} \\
    & = \hat{\mathfrak{R}}_{\mathbf{X}}(\mathcal{F})  - E_{\sigma} \sqbr{\sup_{\norm{\mathbf{Y}}_* \leq C} \frac{1}{|\mathcal{P}|} \sum_{(i,j) \in \mathcal{P}}  2\epsilon \sigma_{ij}\left|\mathbf{X}_{ij} - \mathbf{Y}_{ij} \right| } \nonumber \\
    & \geq \hat{\mathfrak{R}}_{\mathbf{X}}(\mathcal{F}) - \E_{\sigma} \sqbr{\sup_{\norm{\mathbf{Y}}_* \leq C} \frac{2 \epsilon}{|\mathcal{P}|}  \sum_{(i,j) \in \mathcal{P}} \sigma_{ij} \nrbr{\mathbf{Y}_{ij} - \mathbf{X}_{ij}} }   \geq \hat{\mathfrak{R}}_{\mathbf{X}}(\mathcal{F}) - \frac{2 \epsilon C}{|\mathcal{P}|} \E_{\sigma}  \sqbr{ \norm{\bm{\sigma}}} \nonumber
\end{align}
where we have replaced $\sigma_{ij}$ with $-\sigma_{ij}$ as they have the same distribution in the first step and Ledoux-Talagrand contraction lemma in the second step. The subsequent steps are similar to the proof of the upper bound. 

\end{proof}
\subsection{Proof of Theorem \ref{thm:RC_mmmc}}
\begin{proof}
    Using Definition \ref{def:RC_2d}, the empirical Rademacher complexity for the case without any adversarial perturbation is:
    \begin{align*}
        \hat{\mathfrak{R}}_{\mathbf{X}}(\mathcal{F}) = \E_{\sigma} \sqbr{\sup_{\norm{\mathbf{Y}}_* \leq C}\nrbr{ \frac{1}{|\mathcal{P}|} \sum_{(i,j) \in \mathcal{P}} -\sigma_{ij} \mathbf{X}_{ij} \mathbf{Y}_{ij} }} = \E_{\sigma} \sqbr{\sup_{\norm{\mathbf{Y}}_* \leq C}\nrbr{ \frac{1}{|\mathcal{P}|} \sum_{(i,j) \in \mathcal{P}} \sigma_{ij} \mathbf{X}_{ij} \mathbf{Y}_{ij} }}
    \end{align*}
    where we have replaced $\sigma_{ij} $ with $-\sigma_{ij} $ as they have same distribution. 
    
    Similarly, for the case with adversarial perturbation, we can define Rademacher complexity as: 
    \begin{align}
        \hat{\mathfrak{R}}_{\mathbf{X}}(\Tilde{\mathcal{F}}) &= \E_{\sigma} \sqbr{\sup_{\norm{\mathbf{Y}}_* \leq C}\nrbr{ \frac{1}{|\mathcal{P}|} \sum_{(i,j) \in \mathcal{P}} \sigma_{ij} \sup\limits_{\norm{\mathbf{\Delta}_{ij}} \leq \epsilon}  -\nrbr{\mathbf{X}_{ij} + \mathbf{\Delta}_{ij}} \mathbf{Y}_{ij} }} \nonumber \\ 
        & = \E_{\sigma} \sqbr{\sup_{\norm{\mathbf{Y}}_* \leq C}\nrbr{ \frac{1}{|\mathcal{P}|} \sum_{(i,j) \in \mathcal{P}} \sigma_{ij} \min\limits_{\norm{\mathbf{\Delta}_{ij}} \leq \epsilon}  \nrbr{\mathbf{X}_{ij} + \mathbf{\Delta}_{ij}} \mathbf{Y}_{ij} }} \nonumber \\
        & = \E_{\sigma} \sqbr{\sup_{\norm{\mathbf{Y}}_* \leq C}\nrbr{ \frac{1}{|\mathcal{P}|} \sum_{(i,j) \in \mathcal{P}} \nrbr{\sigma_{ij} \mathbf{X}_{ij} \mathbf{Y}_{ij} + \sigma_{ij} \min\limits_{\norm{\mathbf{\Delta}_{ij}} \leq \epsilon} \mathbf{\Delta}_{ij} \mathbf{Y}_{ij} }}} \nonumber\\
        & = \E_{\sigma} \sqbr{\sup_{\norm{\mathbf{Y}}_* \leq C}\nrbr{ \frac{1}{|\mathcal{P}|} \sum_{(i,j) \in \mathcal{P}} \nrbr{\sigma_{ij} \mathbf{X}_{ij} \mathbf{Y}_{ij} - \epsilon  \sigma_{ij} \left| \mathbf{Y}_{ij} \right| }}}\qquad \text{(using Corollary \ref{cor:maxmarginmat_linf})}\label{eq:mmmc_RCsim}\\    
        & \leq \E_{\sigma} \sqbr{\sup_{\norm{\mathbf{Y}}_* \leq C} \frac{1}{|\mathcal{P}|} \sum_{(i,j) \in \mathcal{P}} \sigma_{ij} \mathbf{X}_{ij} \mathbf{Y}_{ij}  } + \E_{\sigma} \sqbr{\sup_{\norm{\mathbf{Y}}_* \leq C} \frac{\epsilon}{|\mathcal{P}|} \sum_{(i,j) \in \mathcal{P}} -\sigma_{ij} |\mathbf{Y}_{ij}|  } \nonumber \\    
        & \leq       \hat{\mathfrak{R}}_{\mathbf{X}}(\mathcal{F}) + \E_{\sigma}  \sqbr{\frac{\epsilon}{|\mathcal{P}|} \sup_{\norm{\mathbf{Y}}_* \leq C} \norm{\bm{\sigma}} \norm{\mathbf{Y}}_*} \nonumber \\
        & = \hat{\mathfrak{R}}_{\mathbf{X}}(\mathcal{F}) + \frac{\epsilon C}{|\mathcal{P}|} \E_{\sigma}  \sqbr{ \norm{\bm{\sigma}}} 
        \end{align}
        where $\bm{\sigma}$ is a matrix containing $\sigma_{ij}$ for entries $(i,j) \in \mathcal{P}$ and $\norm{\cdot}_{\infty}$ is the entry wise $\ell_{\infty}$ norm. This completes the proof for the upper bound. 

For the lower bound, we use $\sup(A+B) \geq \sup(A) -\sup(-B)$ and linearity of expectation in Eq. \eqref{eq:mmmc_RCsim} to arrive at: 
 \begin{align*}
        \hat{\mathfrak{R}}_{\mathbf{X}}(\Tilde{\mathcal{F}}) &\geq \E_{\sigma} \sqbr{\sup_{\norm{\mathbf{Y}}_* \leq C} \frac{1}{|\mathcal{P}|} \sum_{(i,j) \in \mathcal{P}} \sigma_{ij} \mathbf{X}_{ij} \mathbf{Y}_{ij}  } - \E_{\sigma} \sqbr{\sup_{\norm{\mathbf{Y}}_* \leq C} \frac{\epsilon}{|\mathcal{P}|} \sum_{(i,j) \in \mathcal{P}} \sigma_{ij} |\mathbf{Y}_{ij}|  } \nonumber \\
        &\geq \hat{\mathfrak{R}}_{\mathbf{X}}(\mathcal{F}) - \E_{\sigma}  \sqbr{\frac{\epsilon}{|\mathcal{P}|} \sup_{\norm{\mathbf{Y}}_* \leq C} \norm{\bm{\sigma}} \norm{\mathbf{Y}}_*} \\
        & = \hat{\mathfrak{R}}_{\mathbf{X}}(\mathcal{F}) - \frac{\epsilon C}{|\mathcal{P}|} \E_{\sigma}  \sqbr{ \norm{\bm{\sigma}}} 
\end{align*}
The last two steps are similar to the proof of upper bound. This completes the proof for lower bound. 

\end{proof}

\subsection{Proof of Corollary \ref{corr:RC_mmmc_lbtight}}
\begin{proof}
For the lower bound, we substitute $\sigma_{ij}$ as $-\sigma_{ij}$ in Eq. \eqref{eq:mmmc_RCsim} as both have same distribution to arrive at:
\begin{align}
    \hat{\mathfrak{R}}_{\mathbf{X}}(\Tilde{\mathcal{F}}) = \E_{\sigma} \sqbr{\sup_{\norm{\mathbf{Y}}_1 \leq C}\nrbr{ \frac{1}{|\mathcal{P}|} \sum_{(i,j) \in \mathcal{P}} \nrbr{-\sigma_{ij} \mathbf{X}_{ij} \mathbf{Y}_{ij} + \epsilon  \sigma_{ij} \left| \mathbf{Y}_{ij} \right| }}} \label{eq:mmmc_RCminussigma}
\end{align}
Similarly, we substitute $\mathbf{Y}$ as $-\mathbf{Y}$ in the above equation to arrive at: 
\begin{align}
   \hat{\mathfrak{R}}_{\mathbf{X}}(\Tilde{\mathcal{F}}) = \E_{\sigma} \sqbr{\sup_{\norm{\mathbf{Y}}_1 \leq C}\nrbr{ \frac{1}{|\mathcal{P}|} \sum_{(i,j) \in \mathcal{P}} \nrbr{\sigma_{ij} \mathbf{X}_{ij} \mathbf{Y}_{ij} + \epsilon  \sigma_{ij} \left| \mathbf{Y}_{ij} \right| }}} \label{eq:mmmc_RCminusw}
\end{align}
Adding Eq. \eqref{eq:mmmc_RCsim} and Eq. \eqref{eq:mmmc_RCminusw}, we arrive at: 
\begin{align}
    \hat{\mathfrak{R}}_{\mathbf{X}}(\Tilde{\mathcal{F}}) =& \frac{1}{2} \nrbr{\E_{\sigma} \sqbr{\sup_{\norm{\mathbf{Y}}_1 \leq C}\nrbr{ \frac{1}{|\mathcal{P}|} \sum_{(i,j) \in \mathcal{P}} \nrbr{\sigma_{ij} \mathbf{X}_{ij} \mathbf{Y}_{ij} - \epsilon  \sigma_{ij} \left| \mathbf{Y}_{ij} \right| }}}} \nonumber \\
    & +  \frac{1}{2} \nrbr{\E_{\sigma} \sqbr{\sup_{\norm{\mathbf{Y}}_1 \leq C}\nrbr{ \frac{1}{|\mathcal{P}|} \sum_{(i,j) \in \mathcal{P}} \nrbr{\sigma_{ij} \mathbf{X}_{ij} \mathbf{Y}_{ij} + \epsilon  \sigma_{ij} \left| \mathbf{Y}_{ij} \right| }}}}  \nonumber \\
    & \geq \E_{\sigma} \sqbr{\sup_{\norm{\mathbf{Y}}_1 \leq C}\nrbr{ \frac{1}{|\mathcal{P}|} \sum_{(i,j) \in \mathcal{P}} \sigma_{ij} \mathbf{X}_{ij} \mathbf{Y}_{ij} }} = \hat{\mathfrak{R}}_{\mathbf{X}}(\mathcal{F}) \label{eq:RC_mmmc_lb1}
\end{align}
Similarly, adding Eq. \eqref{eq:mmmc_RCminussigma} and Eq. \eqref{eq:mmmc_RCminusw}, we can claim: 
\begin{align}
     \hat{\mathfrak{R}}_{\mathbf{X}}(\Tilde{\mathcal{F}}) \geq \E_{\sigma} \sqbr{\sup_{\norm{\mathbf{Y}}_1 \leq C}\nrbr{ \frac{1}{|\mathcal{P}|} \sum_{(i,j) \in \mathcal{P}} \epsilon  \sigma_{ij} \left| \mathbf{Y}_{ij} \right| }} = \E_{\sigma}  \sqbr{\frac{\epsilon}{|\mathcal{P}|} \sup_{\norm{\mathbf{Y}}_1 \leq C} \norm{\bm{\sigma}}_{\infty}\norm{\mathbf{Y}}_1} = \frac{\epsilon C}{|\mathcal{P}|} \label{eq:RC_mmmc_lb2}
\end{align}
Combining the results from Eq. \eqref{eq:RC_mmmc_lb1} and Eq. \eqref{eq:RC_mmmc_lb2} leads to the desired lower bound. 

\end{proof}
\section{Activation function decomposed as difference of convex functions in Section \ref{sec:twolayernn}: Two-Layer Neural Networks}
\label{sec:diff_conv}
The first step to solve this optimization problem is constructing the functions $g(\mathbf{\Delta})$ and $h(\mathbf{\Delta})$, which requires decomposing the activation functions as the difference of convex functions.  In order to do this, we decompose various activation functions commonly used in the literature as a difference of two convex functions. The decomposition is done by constructing a linear approximation of the activation function around the point where it changes the curvature. These results are presented in Table \ref{tab:diff_conv}. Activation functions like hyperbolic tangent, inverse tangent, sigmoid, inverse square root, and ELU  change the curvature at $z = 0$, and hence it is defined in piece-wise manner. Other functions like GELU, SiLU, and clipped ReLU change the curvature at two points and hence have three ``pieces''. 

It should be noted that $\sigma_1(z)$ and $\sigma_2(z)$ in Table \ref{tab:diff_conv} are proper continuous convex functions which allows us to use the difference of convex algorithms (DCA) \citep{tao1997convex}. By ReLU and variants in Table \ref{tab:diff_conv}, we refer to ReLU, leaky ReLU, parametrized ReLU, randomized ReLU, and shifted ReLU. The decomposition for scaled exponential linear unit (SELU) \citep{klambauer2017self} can be done as shown for ELU in Table \ref{tab:diff_conv}. 

The first five rows of Table \ref{tab:diff_conv} correspond to $\sigma_2(\mathbf{\Delta}) = 0$. It should be noted that $h(\mathbf{\Delta}) \neq 0$ even if $\sigma_2(\mathbf{\Delta}) = 0$, which is evident from Eq. \eqref{eq:h_def}. Hence the original problem may not be convex, and we may have to use the difference of convex programming approach even for activation functions with $\sigma_2(\mathbf{\Delta}) = 0$.

\begin{table}
	\caption{Activation function decomposed as a difference of convex functions. }
	\label{tab:diff_conv}
	\centering
	{\scriptsize
	\begin{tabular}{@{}l@{\hspace{0.05in}}l@{\hspace{0.05in}}l@{\hspace{0.05in}}l@{\hspace{0.05in}} l@{\hspace{0.05in}}}
		\toprule
		Name&$\sigma(z)$ & $\sigma_1(z)$& $\sigma_2(z)$& Domain \\ 
		\midrule
		Linear& $z$& $z$& $0$ & $\mathbb{R}$\\ 
		Softplus& $\log(1+e^z)$& $\log(1+e^z)$& $0$  & $\mathbb{R}$\\ 
		ReLU \& variants & $\max(0,z)$& $\max(0,z)$ & $0$  & $\mathbb{R}$\\ 
		Bent Identity & $\frac{\sqrt{x^2 + 1} - 1}{2} + x$& $\frac{\sqrt{x^2 + 1} - 1}{2} + x$& $0$  & $\mathbb{R}$\\ 
		\begin{tabular}{@{}c@{}}Inverse square root \\ linear unit\end{tabular}   & $\begin{cases}\frac{z}{1+az^2} \\ z  \end{cases}$& $\begin{cases}\frac{z}{1+az^2} \\ z  \end{cases}$ & $0$ & $\begin{cases} z<0 \\ z \geq 0\end{cases}$ \\ 
		\begin{tabular}{@{}c@{}}Hyperbolic  \\ Tangent\end{tabular} & $\tanh(z) $& $\begin{cases}\tanh(z) - z \\ z \end{cases}$ & $\begin{cases} -z \\ z - \tanh(z) \end{cases}$ & $\begin{cases} z<0 \\ z \geq 0\end{cases}$\\ 
		\begin{tabular}{@{}c@{}}Inverse  \\ Tangent\end{tabular}   & $\arctan(z)$& $\begin{cases}\arctan(z) - z \\ z  \end{cases}$ & $\begin{cases}-z \\ z - \arctan(z)  \end{cases}$ & $\begin{cases} z<0 \\ z \geq 0\end{cases}$\\ 
		Sigmoid& $\frac{1}{1+\exp(-z)}$& $\frac{1}{2} \begin{cases}\tanh(z/2) + 1 - z/2\\ z/2 + 1 \end{cases}$& $ \frac{1}{2}  \begin{cases} -z/2 \\ z/2 - \tanh(z/2)  \end{cases}$& $\begin{cases} z<0 \\ z \geq 0\end{cases}$ \\
		\begin{tabular}{@{}c@{}}Gauss Error \\ Function\end{tabular} & $\frac{2}{\sqrt{\pi}}\int_{0}^ze^{-t^2}dt$& $\frac{2}{\sqrt{\pi}} \begin{cases}\int_{0}^ze^{-t^2}dt - z \\ z \end{cases} $ & $\frac{2}{\sqrt{\pi}} \begin{cases} -z  \\ z - \int_{0}^ze^{-t^2}dt \end{cases} $ & $\begin{cases} z<0 \\ z \geq 0\end{cases}$ \\ 
		\begin{tabular}{@{}c@{}c@{}}Gauss Error \\ Linear Unit \\  (GELU)\end{tabular} & $\frac{z}{2}\nrbr{1 + \frac{2}{\sqrt{\pi}}\int_{0}^ze^{-t^2}dt} $& $ \begin{cases}-0.13z - 0.29 \\ \sigma(z) \\0.13z - 0.29 \end{cases} $ & $ \begin{cases} -0.13z - 0.29 - \sigma(z) \\ 0 \\ 0.13z - 0.29 - \sigma(z) \end{cases} $  & $\begin{cases} z<-\sqrt{2} \\-\sqrt{2} \leq z \leq \sqrt{2} \\ z \geq \sqrt{2}\end{cases}$   \\ 
		\begin{tabular}{@{}c@{}}Inverse square \\ root Unit \end{tabular} & $\frac{z}{\sqrt{1+az^2}}$& $\begin{cases}\frac{z}{1+az^2} -z\\ z  \end{cases}$ & $\begin{cases}-z \\ z - \frac{z}{1+az^2}  \end{cases}$ & $\begin{cases} z<0 \\ z \geq 0\end{cases}$ \\
	\begin{tabular}{@{}c@{}}Sigmoid Linear \\ Unit  (SiLU)\end{tabular} & $\frac{z}{1+\exp(-z)}$& $ \begin{cases}-0.01z - 0.44 \\ \sigma(z) \\1.01z - 0.44 \end{cases} $ & $ \begin{cases} -0.01z - 0.44 - \sigma(z) \\ 0 \\ 1.01z - 0.44 - \sigma(z) \end{cases} $  & $\begin{cases} z<-2.4 \\-2.4 \leq z \leq 2.4 \\ z \geq 2.4\end{cases}$   \\ 
	\begin{tabular}{@{}c@{}}Exponential Linear \\ Unit (ELU) \end{tabular} & $\begin{cases} \alpha (e^z - 1) \\ z\end{cases}$& $\begin{cases}\alpha (e^z - 1) \\ \alpha z  \end{cases}$ & $\begin{cases}0 \\ (\alpha-1)z  \end{cases}$ & $\begin{cases} z<0 \\ z \geq 0\end{cases}$ \\
	Clipped RELU& $\begin{cases} 0 \\ z \\ a \end{cases}$& $\max(z,0)$& $\max(z-a, 0)$ & $\begin{cases}z \leq 0 \\ 0\leq z\leq a \\ z \geq a \end{cases}$\\ 
	\bottomrule
	\end{tabular}}
\end{table}

\section{Additional Details for Section \ref{sec:exp}: Real-World Experiments}
\label{sec:appn_exp}
In this section, we validate the proposed method for various ML problems. Our intention in this work is not motivated towards designing the most optimal algorithm which solves all the ML problems discussed previously. Rather, we demonstrate the practical utility of our novel approach that can integrate our plug-and-play solution with widely accepted ML models. As a generic optimization algorithm, we chose to work with projected gradient descent in all the experiments, which can be replaced with any other variant of the user's choice. 

We compare our proposed approach with two other training approaches.  The first comparison method is the classical approach of training without any perturbation, meaning $\mathbf{\Delta}^{\star} = \mathbf{0}$ in Algorithm \ref{alg:main}. The second approach is to directly use a random $\mathbf{\Delta}^{\star}$ without solving the optimization problem. These two comparison methods are referred as ``No error'' and ``Random'' in the columns of Table \ref{tab:all_results} and Table \ref{tab:all_exp_time_new}. We also compared the proposed approach against the fast gradient sign method (FGSM) \citep{goodfellow2014explaining}, projected gradient descent (PGD), and TRADES \citep{zhang2019theoretically}. While we could have used $\epsilon$-perturbed test data (coming from the same distribution of the training data) with some synthetic adversary, we preferred to use a more challenging scenario: test data coming from a different distribution than the training data. Our proposed approach is referred as ``Proposed''. As there are different ML problems, we use different performance metrics for comparison. 

\begin{table}
	\caption{Running time (in seconds) for experiments on real-world data sets for various supervised and unsupervised ML problems. Note that the running times of the proposed approach is comparable (most of the time) to the one of other methods (``No error'', ``Random'', ``FGSM'', ``PGD'',``TRADES'')}.
	\label{tab:all_exp_time_new}
	\centering
	{\scriptsize
	\begin{tabular}{@{}l@{\hspace{0.1in}}l@{\hspace{0.05in}}l@{\hspace{0.05in}}l@{\hspace{0.05in}}l@{\hspace{0.05in}}|l@{\hspace{0.05in}}l@{\hspace{0.05in}}l@{\hspace{0.05in}}l@{\hspace{0.05in}}l@{\hspace{0.05in}}|l@{\hspace{0.05in}}}
		\toprule
		&Problem & Loss function& Dataset & Norm & No error & Random & FGSM & PGD & TRADES &Proposed \\
		\midrule
		\multirow{6}{*}{\begin{turn}{90}  Warm up\end{turn}} &Regression& Squared loss & BlogFeedback & Euclidean & 5.66  & 19.41  & 10.5&33.9 &46.2  & 6.03 \\ 
		&Regression & Squared loss & BlogFeedback & $\ell_{\infty}$& 5.76  & 19.64  & 10.3&33.9 &46.2  & 6.15\\
		&Classification & Logistic loss & ImageNet & Euclidean & 22.69 &  64.12 &  48.3& 145.4&147.2 & 21.8 \\
		&Classification & Logistic loss & ImageNet & $\ell_{\infty}$&  22.64  & 64.39  & 49.7& 158.5&124.2 &21.75\\
		&Classification & Hinge loss & ImageNet & Euclidean & 20.03 &  59.82  & 47.1 & 139.4& 108.1 & 18.29\\
		&Classification & Hinge loss & ImageNet & $\ell_{\infty}$& 20.52 &  59.5  & 49.2& 146.3&127.3 &18.06 \\ [0.1 cm] \hline
           \multirow{8}{*}{\begin{turn}{90}  Main results\end{turn}}  &Classification & NN: ReLU & ImageNet & Euclidean &128.8&  159.4 &   191.2 & 892.7  & 801.1  & 951.7\\
            &Classification & NN: Sigmoid & ImageNet & Euclidean &132.6  &  140.1   & 172.8 & 831.8 & 687.2 &  1606\\
		&Graphical Model & Log-likelihood & TCGA & Euclidean & 5.1  &  5.43  &  8.9& 13.7 &10.1 &5.82\\
		&Graphical Model & Log-likelihood & TCGA & $\ell_{\infty}$ &  3.47  &  3.86 &  5.58  & 10.91  & 11.39 & 6.85\\
		&Matrix Completion & Squared loss & Netflix & Frobenius & 9.85  & 10.93 &  12.7 &15.1 & 15.8 & 10.94\\
		&Matrix Completion & Squared loss & Netflix & Entry-wise $\ell_{\infty}$& 9.82  & 10.42  & 13.7 &13.7 & 14.74 &10.54 \\
		&Max-Margin MC & Hinge loss  &HouseRep & Frobenius & 7.3  &  7.99  &  9.45 &  11.69  & 11.59 &   8.72\\
		&Max-Margin MC & Hinge loss  &HouseRep & Entry-wise $\ell_{\infty}$ &7.31  &  8.19 &   9.72 &  12.98  & 13.43  &  9.88 \\
		\bottomrule
	\end{tabular}}
\end{table}

\paragraph{Regression:} We consider the \href{https://archive.ics.uci.edu/ml/datasets/BlogFeedback}{BlogFeedback} dataset to predict the number of comments on a post. We chose the first 5000 samples for training and the last 5000 samples for testing, where the samples were arranged as per time. Hence, the test data corresponds to the future in correspondence to the training set. We perform training using the three approaches. In our method, we plug $\mathbf{\Delta}^{\star}$ using Theorem \ref{thm:reg_sqerr}. We use mean square error (MSE) to evaluate the performance on a test set, which is reported to be the lowest for our proposed approach for the Euclidean and $\ell_\infty$ norm constraints.

In the experiments for minimizing the loss function, we did $500$ GD iterations with step size at $j^{\text{th}}$ iteration as $\eta_j = \frac{0.1}{\sqrt{j}}$, perturbation budget of $\epsilon = 0.1$ for Euclidean and infinity norm constraint. In the inner maximization defined in Eq. \eqref{eq:l2loss_lpconst}, we perform $5$ iterations of PGD and TRADES. 

\paragraph{Classification:} For this task, we use the \href{https://image-net.org/}{ImageNet} dataset which is available publicly. The dataset contains $1000$ bag-of-words features. For training, we used ``Hungarian pointer'' having $2334$ samples versus ``Lion'' with $1795$ samples. For testing, we used ``Siamese cat'' with  1739 samples versus  ``Tiger''  having 2086 samples. We train the model with the logistic loss by supplying $\mathbf{\Delta}^{\star}$ from Theorem \ref{thm:logit} for our algorithm. We use the accuracy metric  to evaluate the performance on a test set, and it was observed to be the best for our proposed algorithm. The same procedure was applied to test the hinge loss function by supplying  $\mathbf{\Delta}^{\star}$ from Theorem \ref{thm:hinge} and we note that our proposed algorithm performs better than the other approaches for the Euclidean and $\ell_\infty$ norm constraints. The parameters used in training for various models is as follows:
\begin{enumerate}
\item Logistic regression and hinge Loss: In the experiments for minimizing the loss function, we did $500$ GD iterations with step size at $j^{\text{th}}$ iteration as $\eta_j = \frac{0.1}{\sqrt{j}}$, perturbation budget of $\epsilon = 3$ for Euclidean and infinity norm constraint. In the inner maximization defined in Eq. \eqref{eq:logit_lpconst} or Eq. \eqref{eq:hinge_lpconst}, we perform $5$ iterations of PGD and TRADES.
\item Two-layer neural networks: In the experiments for minimizing the loss function, we did $150$ GD iterations with step size at $j^{\text{th}}$ iteration as $\eta_j = \frac{0.1}{\sqrt{j}}$, perturbation budget of $\epsilon = 0.005$ for Euclidean norm constraint. In the inner maximization defined in Eq. \eqref{eq:2layer_gen1}, we perform $15$ iterations of PGD and TRADES. In our CCCP approach, we used $3$ outer GD iterations and $5$ inner GD iterations for fixed $h(\mathbf{\Delta})$.

\end{enumerate}

\paragraph{Gaussian Graphical models:}  For this task, we use the publicly available Cancer Genome Atlas \href{http://tcga-data.nci.nih.gov/tcga/}{(TCGA)} dataset. The dataset contains gene expression data for 171 genes. We chose breast cancer (590 samples) for training and ovarian cancer (590 samples) for testing. The adversarial perturbation for robust learning in the proposed method was supplied from Theorem \ref{thm:Gaussmdl_l2} and Theorem \ref{thm:Gaussmdl_linf}. We compare the training approaches based on the log-likelihood of a test set from the learned precision matrices. The log-likelihood is reported to be the largest for our proposed approach for the Euclidean and entry-wise $\ell_\infty$ norm constraints.

In the experiments for minimizing the loss function, we did $100$ GD iterations with step size at $j^{\text{th}}$ iteration as $\eta_j = \frac{0.1}{\sqrt{j}}$, perturbation budget of $\epsilon = 2$ and $\epsilon = 0.05$ for Euclidean and infinity norm constraint respectively. In the inner maximization defined in Eq. \eqref{eq:Graph_lpconst1}, we perform $5$ iterations of PGD and TRADES. 

\paragraph{Matrix Completion:} For this problem, we use the publicly available \href{https://www.kaggle.com/datasets/netflix-inc/netflix-prize-data}{Netflix Prize} dataset. We chose the 1500 users and 500 movies with most ratings. We randomly assigned the available user/movie ratings to the training and testing sets. As the users can be from any location, age, gender, nationality and movies can also have different genres, language, or actors. Training was performed using the six approaches. In our method, $\mathbf{\Delta}^{\star}$ was chosen using Theorem \ref{thm:mat_compl_fro} and Corollary \ref{cor:matrixcompl_linf}. We use the MSE metric on the test set which is reported to be the lowest for our proposed method. 

In the experiments for minimizing the loss function, we did $100$ GD iterations with step size at $j^{\text{th}}$ iteration as $\eta_j = \frac{0.1}{\sqrt{j}}$, perturbation budget of $\epsilon = 3000$ and $3$ for Frobenius norm constraint and entry-wise $\ell_{\infty}$ constraint respectively. In the inner maximization defined in Eq. \eqref{eq:mat_compl_1}, we perform $5$ iterations of PGD and TRADES. 

\paragraph{Max-Margin Matrix Completion:} We used the votes in the House of Representatives \href{https://www.govtrack.us/congress/votes}{(HouseRep)}  for the first session of the $110^{\text{th}}$ U.S. congress. The HouseRep dataset contains 1176 votes for   435 representatives. A ``Yea'' vote was considered +1, a ``Nay'' vote was considered -1. We randomly assigned the available votes to the training and testing sets. To test our adversarial training framework, the training set had a non-uniform distribution of ``Yea'' or ``Nay'' vote, whereas, for the test set, we chose it to be uniform random. Further, we trained the model using the six approaches. In our method, $\mathbf{\Delta}^{\star}$ was chosen using Theorem \ref{thm:maxmarginmat_fro} and Corollary \ref{cor:maxmarginmat_linf}. We use the percentage of correctly recovered votes on a test set as the metric to compare the three training approaches. 

In the experiments for minimizing the loss function defined in Eq. \eqref{eq:lossfn_maxmargin}, we did $100$ GD iterations with step size at $j^{\text{th}}$ iteration as $\eta_j = \frac{0.1}{\sqrt{j}}$, perturbation budget of $\epsilon = 4500$ and $75$ for Frobenius norm constraint and entry-wise $\ell_{\infty}$ constraint respectively. In the inner maximization defined in Eq. \eqref{eq:maxmarginmat_prbdef}, we perform $5$ iterations of PGD and TRADES. 

All the codes were ran on a machine with an 2.2 GHz Quad-Core Intel Core i7 processor with RAM of 16 GB. The running times for the three training approaches is compared in Table \ref{tab:all_exp_time_new}. It can be seen that the running time for the proposed approach is comparable to the other training approaches for most of the ML problems discussed above.

\end{document}